\documentclass{article}
\usepackage{kr}
\usepackage{times}
\usepackage{soul}
\usepackage{url}
\usepackage[hidelinks]{hyperref}
\usepackage[utf8]{inputenc}
\usepackage[small]{caption}
\usepackage{graphicx}
\usepackage{amsmath}
\usepackage{amsthm}
\usepackage{booktabs}
\hypersetup{colorlinks=false,linkcolor=blue,urlcolor=blue}

\newtheorem{example}{Example}
\newtheorem{theorem}{Theorem}

\usepackage[dvipsnames]{xcolor}
\usepackage{tkz-graph}
\usepackage{amsmath,amsthm}
\usepackage{amsfonts}
\usepackage{amssymb}
\usepackage{natbib}

\newcommand{\unit}{\ensuremath{\mathbf{U}}}

\newcommand{\kb}{\ensuremath{\mathcal{K}}}
\newcommand{\db}{\ensuremath{D}}
\newcommand{\marco}{\textcolor{black}}

\newcommand{\C}{\ensuremath{\mathsf{C}}}
\newcommand{\PS}{\ensuremath{\mathsf{P}}}
\newcommand{\V}{\ensuremath{\mathsf{V}}}

\newcommand{\Dom}{\ensuremath{\mathbb{D}}}

\newcommand{\WCQ}{\ensuremath{\mathbb{NXL}}}

\usepackage{pifont}

\newcommand{\FreeConst}{\mathit{Fr}}
\newcommand{\Genes}{\mathit{Ge}}

\newcommand{\nop}[1]{}

\usepackage{graphicx} 
\usepackage{paralist} 
\usepackage{stmaryrd} 
\usepackage{multicol} 
\newtheorem{cor}{Corollary}

\usepackage{relsize}

\usepackage[linesnumbered,ruled,vlined]{algorithm2e}
\SetKwInput{KwInput}{Input}                
\SetKwInput{KwOutput}{Output}              
\newtheorem{proposition}{Proposition}
\newtheorem{lemma}{Lemma}
\newtheorem{definition}{Definition}

\newtheorem{remark}{Remark}

\title{Characterizing Nexus of Similarity within Knowledge Bases:\\A Logic-based Framework and its Computational Complexity Aspects}

\author{%
Giovanni Amendola$^1$\and
Marco Manna$^1$\and
Aldo Ricioppo$^1$\and
\affiliations
$^1$University of Calabria\\
\emails
\{giovanni.amendola, marco.manna, aldo.ricioppo\}@unical.it
}

\begin{document}
\maketitle

\begin{abstract}
Similarities between entities occur frequently in many real-world scenarios. For over a century, researchers in different fields have proposed a range of approaches to measure the similarity between entities. More recently, inspired by ``Google Sets'', significant academic and commercial efforts have been devoted to expanding a given set of entities with  similar ones. As a result, existing approaches nowadays are able to take into account properties shared by entities, hereinafter called nexus of similarity. Accordingly, machines are largely able to deal with both similarity measures and set expansions. To the best of our knowledge, however, there is no way to characterize nexus of similarity between entities, namely identifying such nexus in a formal and comprehensive way so that they are both machine- and human-readable; moreover, there is a lack of consensus on evaluating existing approaches for weakly similar entities. 
As a first step towards filling these gaps, we aim to complement existing literature by  developing a novel logic-based framework to formally and automatically characterize nexus of similarity between tuples of entities within a knowledge base. Furthermore, we analyze computational complexity aspects of this framework.
\end{abstract}


\section{Introduction}

Real-world scenarios often involve similarities between entities.
Even at a young age, kids start learning how to informally describe, classify, and compare simple groups of entities. For example, they can easily recognize that both $\langle\mathsf{Paris}\rangle$ and $\langle\mathsf{Rome}\rangle$ are ``cities'', and that $\langle\mathsf{Paris}\rangle$ is more similar to $\langle\mathsf{Rome}\rangle$ than to $\langle\mathsf{Gardaland}\rangle$.
Growing up, people generally get better at identifying properties shared by entities, hereinafter called {\em nexus of similarity}.%
\footnote{Note that the term {\em similarity} is often misused to refer the broader notion of {\em relatedness}~\cite{DBLP:journals/csur/ChandrasekaranM21}.}
Adults, for example, can easily agree that $\langle\mathsf{Paris}\rangle$ and $\langle\mathsf{Rome}\rangle$ are both ``Europe's capitals situated on rivers''.
Likewise, one can identify nexus of similarity between \mbox{$n$-ary} tuples of entities. 
For instance, $\langle\mathsf{Tokyo}, \mathsf{Tokyo~Tower} \rangle$ and $\langle\mathsf{Paris}, \mathsf{Eiffel~Tower} \rangle$ are fairly similar, as each of them is a ``capital paired with one of its famous monument being a tower made of metal''.

As we move towards more complex entities such as goods, services, behaviors, or health conditions, the challenges become greater, and the stakes become more interesting and valuable.
For example, in e-commerce, Amazon's recommendation system may suggest specific smartphones to a user interested in high-end devices equipped with features like accelerometer, compass, fingerprint, and Android 13.
In the tourism industry, travel agencies may recommend other theme parks in the US to a family based on their previously visited ones.
Streaming services such as Netflix suggest trailers to their customers based on their previous viewing history and preferences.
In medicine, researchers may need to understand why certain individuals are more susceptible to certain diseases than others.

For over a century, researchers from various  fields have proposed a range of approaches to {\em semantic similarity}~\cite{GomaaFahmy13} ---a problem having as expected output a similarity measure between entities, usually in the form of a descriptive rating or a numerical score.
Moreover, in the past two decades, inspired by ``Google Sets''~\cite{GoogleSets08}, considerable academic and commercial efforts have been devoted to providing solutions for expanding a given set of entities with similar ones. The most studied tasks here are
{\em entity set expansion}~\cite{DBLP:conf/emnlp/PantelCBPV09},
{\em entity recommendation}~\cite{DBLP:conf/semweb/BlancoCMT13},
{\em tuples expansion}~\cite{DBLP:conf/iiwas/ErAB16}, or
{\em entity suggestion}~\cite{DBLP:conf/ijcai/ZhangXHWWW17}.
Traditional approaches mainly use  (hyper)text corpora or tabular data as background knowledge~\cite{DBLP:conf/ijcai/GabrilovichM07,DBLP:conf/www/HeX11,DBLP:journals/corr/abs-1301-3781,DBLP:conf/ijcai/XunLZGZ17,DBLP:conf/naacl/DevlinCLT19}.
%
Over the years, more and more approaches have started exploiting structured  knowledge bases (KBs)
automatically enriched or completed from text corpora~\cite{DBLP:conf/aaai/LinLSLZ15,DBLP:conf/aaai/LiuZ0WJD020,DBLP:conf/www/LissandriniMPV20,DBLP:journals/fcsc/ShiDCHWL21};
%
think, for example, at
BabelNet~\cite{DBLP:journals/ai/NavigliP12},
DBpedia~\cite{DBLP:journals/semweb/LehmannIJJKMHMK15},
Wikidata~\cite{DBLP:journals/cacm/VrandecicK14}, and
YAGO~\cite{DBLP:journals/ai/HoffartSBW13}--- typically available in the form of  knowledge graphs (KGs)~\cite{DBLP:journals/csur/HoganBCdMGKGNNN21} and at the core of the Linked Open Data Cloud~\cite{DBLP:journals/semweb/McCraeMHB15}.
Essentially, the heterogeneity, semantic richness and large-scale nature of KBs make traditional \mbox{approaches less effective~\cite{DBLP:conf/www/MaK15}.}
\newcommand{\tmpframea}[1]{\colorbox{Cyan!5}{#1}}

\normalsize

\begin{figure*}[t!]
	\centering
	\begin{minipage}{0.65\textwidth}
		\centering
		\begin{tikzpicture}[x=2.5 cm,y=1.83 cm]
			\SetUpEdge[lw = 0.5pt, color = black, labeltext = black!75, labelcolor = white]
			
			\GraphInit[vstyle=Normal] 
			
			\tikzset{VertexStyle/.style = {
					shape = rectangle, draw, color=black!75,
					font=\scriptsize\sffamily,
					fill = LimeGreen!25}
			}
			
			\Vertex{Discovery Cove}
			\SO(Discovery Cove){Florida}
			\SOWE(Discovery Cove){US}
			\tikzset{VertexStyle/.style = {
					shape = rectangle, draw, color=black!75,
					font=\scriptsize\sffamily,
					fill = LimeGreen!25}
			}
			\EA(Discovery Cove){theme park}
			\NOEA(theme park){amusement park}
			
			\tikzset{VertexStyle/.style = {
					shape = rectangle, draw, color=black!75,
					font=\scriptsize\sffamily,
					fill = LimeGreen!25}
			}
			\SOEA(Discovery Cove){Epcot}
			\EA(amusement park){Prater}
			\SO(Prater){Austria}
			\WE(Discovery Cove){California}
			\NOWE(Discovery Cove){Pacific Park}
			\SO(amusement park){Gardaland}
			\SO(Gardaland){Italy}
			\SO(Austria){Leolandia}
			
			\tikzset{EdgeStyle/.style={->}}
			\tikzset{EdgeStyle/.append style={font=\sffamily\scriptsize}}
			
			\Edge[label=located](Discovery Cove)(Florida)
			\tikzset{EdgeStyle/.style={->, font=\sffamily\scriptsize}}
			\Edge[label=partOf](Florida)(US)
			\tikzset{EdgeStyle/.style={->, font=\sffamily\scriptsize}}
			\Edge[label=isa](Discovery Cove)(theme park)
			\tikzset{EdgeStyle/.style={->,  font=\sffamily\scriptsize}}
			\Edge[label=isa](theme park)(amusement park)
			\Edge[label=located](Epcot)(Florida)
			\Edge[label=isa](Epcot)(theme park)
			\tikzset{EdgeStyle/.style={->, font=\sffamily\scriptsize}}
			\Edge[label=isa](Prater)(amusement park)
			\Edge[label=located](Prater)(Austria)
			\Edge[label=located](Pacific Park)(California)
			\Edge[label=partOf](California)(US)
			\Edge[label=isa](Pacific Park)(amusement park)
			\Edge[label=located](Gardaland)(Italy)
			\Edge[label=isa](Leolandia)(amusement park)
			\Edge[label=located](Leolandia)(Italy)
			\tikzset{EdgeStyle/.append style = {bend left}}
			\Edge[label=isa](Gardaland)(theme park)
			\tikzset{EdgeStyle/.style={->, font=\sffamily\scriptsize}}
			\tikzset{EdgeStyle/.append style={font=\sffamily\scriptsize}}
			\Edge[label=isa](Discovery Cove)(amusement park)
			\Edge[label=isa](Gardaland)(amusement park)
			\tikzset{EdgeStyle/.style={->, font=\sffamily\scriptsize}}
			\Edge[label=isa](Epcot)(amusement park)
		\end{tikzpicture}\vspace{-1mm}
		\caption{Knowledge graph $\mathcal{G}_0$ used in Example~\ref{ex:running} to build the selective knowledge base $\kb_0$.}\vspace{-2mm}
		\label{fig:KG1}  
	\end{minipage}%
	\begin{minipage}{0.36\textwidth}
		\centering
		\begin{tikzpicture}[x=1.35 cm,y=0.9 cm]
			\SetUpEdge[lw = 0.5pt, color = black!75, labeltext = black!75, labelcolor = white]
			\GraphInit[vstyle=Normal] 
			\tikzset{VertexStyle/.style = {
					shape = rectangle, draw, color=black!75,
					font=\scriptsize\sffamily,
					fill = Lavender!20}
			}
			\Vertex[Math,L=\!\bar{\varphi}_a\mapsto\{\langle\textsf{Discovery~Cove}\rangle{\tt,}\langle\textsf{Epcot}\rangle\}\,{\tt=}\,\unit_0]{Discovery Cove; Epcot}
			\NOWE[Math,L=\!\bar{\varphi}_b\mapsto\{\langle\textsf{Pacific~Park}\rangle\}\!](Discovery Cove; Epcot){Pacific Park}
			\NOEA[Math,L=\!\bar{\varphi}_c \mapsto\{\langle\textsf{Gardaland}\rangle\}\!](Discovery Cove; Epcot){Gardaland}
			\NOWE[Math,L=\!\bar{\varphi}_d\mapsto\{\langle\textsf{Prater}\rangle{\tt,} \langle\textsf{Leolandia}\rangle\}\!](Gardaland){Prater; Leolandia}
			\NO[Math,L=\!\bar{\varphi}_e\mapsto\{\langle\textsf{theme park}\rangle\}\!](Prater; Leolandia){theme park}
			\NO[Math,L=\!\bar{\varphi}_f \mapsto\{\langle \xi \rangle~:~\xi\textit{ is any other entity}\}\!](theme park){any other entity}
			\tikzset{EdgeStyle/.style={->}}
			\tikzset{EdgeStyle/.append style={font=\sffamily\scriptsize}}
			\Edge[label=](Discovery Cove; Epcot)(Pacific Park)
			\Edge[label=](Discovery Cove; Epcot)(Gardaland)
			\Edge[label=](Gardaland)(Prater; Leolandia)
			\Edge[label=](Pacific Park)(Prater; Leolandia)
			\Edge[label=](Prater; Leolandia)(theme park)
			\Edge[label=](theme park)(any other entity)
		\end{tikzpicture}\vspace{-1mm}
		\caption{Expansion graph $\mathit{eg}(\unit_0,\kb_0)$.} \vspace{-2mm}
		\label{fig:KG2}  
	\end{minipage}
\end{figure*}
\newcommand{\tmpframe}[1]{\colorbox{Yellow!10}{#1}}

\begin{figure*}[t!]
\footnotesize
\begin{equation*} 
\boxed{
    \begin{split}
        \bar{\varphi}_a  =  x \leftarrow  &  \ 
        \mathsf{isa}(x, \mathsf{tp}),
        \mathit{conj}(\bar{\varphi}_e),
        \mathsf{located}(x,\mathsf{Florida}),
        \hspace{1.8cm}\\
          &  \ 
        \mathsf{partOf}(\mathsf{Florida},\mathsf{US}), 
        \top(\mathsf{tp}),
        \top(\mathsf{Florida}), \top(\mathsf{US}) \\
    \bar{\varphi}_b =   x  \leftarrow & \  \mathsf{isa}(x,\mathsf{ap}),  \mathsf{located}(x,y), \top(x),  \top(y), \top(\mathsf{ap}),\\
    & \ \mathsf{partOf}(y,\mathsf{US}),  \top(\mathsf{US})
    \end{split}
    \begin{split}
        \bar{\varphi}_c  =  x  \leftarrow & \ \mathsf{isa}(x, \mathsf{tp}), \mathit{conj}(\bar{\varphi}_d), \top(\mathsf{tp})\\
        \bar{\varphi}_d = x  \leftarrow & \  \mathsf{isa}(x,\mathsf{ap}),  \mathsf{located}(x,y), \top(x),  \top(y), \top(\mathsf{ap})\\
        \bar{\varphi}_e  = x  \leftarrow & \  \mathsf{isa}(x,\mathsf{ap}), \top(x), \top(\mathsf{ap})\\
        \bar{\varphi}_f =  x \leftarrow & \ \top(x)\\
    \end{split}
    }
    \end{equation*}

\vspace{-4mm}
\caption{Characterizations occurring in Figure~\ref{fig:KG2}, where $\mathsf{tp}$ stands for $\mathsf{theme~park}$, $\mathsf{ap}$  for $\mathsf{amusement~park}$, and $\mathit{conj}(\varphi)$ for  the atoms of $\varphi$.} 
\vspace{-3mm}\label{fig:queries}
\end{figure*}
\normalsize

Nowadays, computing machines have advanced to the point where they can perform various tasks related to entity similarity, such as: 
$(a)$ classifying $\langle\mathsf{Paris}\rangle$ and $\langle\mathsf{Rome}\rangle$ as ``cities'' or ``capitals''; 
$(b)$ calculating proper similarity scores by taking into account that both entities are ``European capitals situated on rivers''; and 
$(c)$ expanding these entities to include $\langle\mathsf{Bratislava}\rangle$, $\langle\mathsf{Budapest}\rangle$, $\langle\mathsf{Prague}\rangle$, and $\langle\mathsf{Vienna}\rangle$, by following the same rationale.
%
%
To deal with similarities, researchers agree that an entity, whether a real-world object or an abstract concept, can be described by a set of {\em properties} that carry relevant information about it and form {\em semantic connections} with other entities~\cite{DBLP:reference/db/Balog18}.
A selection of these properties is often referred to as a {\em summary} of that entity~\cite{DBLP:journals/ws/LiuCGQ21}.
Both humans and machines have their own summaries available for a number of entities, which they rely on while dealing with them.
For humans, these summaries essentially collect the most prominent properties  they remember about the entities they know.
For machines, such summaries consist of the available properties that are artificially encoded or embedded, describing the entities they can access within text corpora or KBs.

As far as we know, there are two major gaps in the presented  picture. On the one hand, there is no way to formally and comprehensively characterize the nexus of similarity between tuples of entities so that they are both human- and machine-readable. For instance, one could use formal (onto)logical expressions~\cite{DBLP:reference/fai/LifschitzMP08} or description logics~\cite{DBLP:conf/dlog/2003handbook} to express the nexus of similarity between $\langle\mathsf{Paris}\rangle$ and $\langle\mathsf{Rome}\rangle$. 
On the other hand, there is no consensus on how to evaluate existing approaches when entities are weakly similar~\cite{DBLP:conf/naacl/AsrZJ18,DBLP:conf/cogsci/DeyneNPS12}.
While people and machines may agree on the similarity of entities like $\langle\mathsf{Paris}\rangle$ and $\langle\mathsf{Rome}\rangle$, 
disagreements may arise in cases like $\langle\mathsf{Paris}\rangle$ and $\langle\mathsf{Gardaland}\rangle$ due to random differences in summaries. 
As a result, benchmarks for similarity may diverge, making it difficult for machines to receive fair evaluations.

The present paper aims to address the aforementioned gaps by proposing a formal reference semantics to deal with nexus of similarity. As a first step, we develop a logic-based framework to characterize the nexus of similarity between (tuples of) entities within a KB. This not only provides value in itself but also ensures that machines can compare with each other in a fair and formal manner, besides human opinion.
To illustrate the key principles of our framework, we provide a running example.

\begin{example}\label{ex:running}
Let's say we start with an initial KB represented by the KG $\mathcal{G}_0$ shown in Figure~\ref{fig:KG1}.
This graph can be encoded as a dataset $D_0$ of the form $\{\mathsf{located}(\mathsf{Epcot}$,
$\mathsf{Florida})$,
$\mathsf{partOf}(\mathsf{Florida},\mathsf{US})$,
$\mathsf{isa}(\mathsf{Epcot},\mathsf{tp}),
\mathsf{isa}(\mathsf{Epcot},\mathsf{ap}), \hspace{1.3mm}
...,$
$\top(\mathsf{Epcot})$,
$\top(\mathsf{Florida})$,
$\top(\mathsf{US})$,
$\top(\mathsf{tp})$,
$\top(\mathsf{ap})$,
...$\}$, where $\mathsf{tp}$ stands for $\mathsf{theme~park}$ and $\mathsf{ap}$  for $\mathsf{amusement~park}$.
%
%
To have a summary for any entity of $\mathcal{G}_0$, we need to use an algorithm that is tailored to the considered scenario, called  summary selector.
%
We now consider the simple yet effective selector $\varsigma_0$ that builds, for each entity $e$ of $D_0$, the dataset $\varsigma_0(\langle e \rangle)$ as the union of 
$A(e) = \{\mathsf{isa}(e,e'): \mathsf{isa}(e,e') \in D_0\}$,
$B(e) = \{p(e,e') : p(e,e') \in D_0 \wedge p \neq \mathsf{isa}\}$,
$C(e) =\{p'(e',e'') : p'(e',e'') \in D_0 \wedge p(e,e') \in B \wedge p' \neq \mathsf{isa}\}$ , and
$D(e) = \{\top(e') : e'\mbox{ occurs in }A(e)\cup B(e) \cup C(e)\}$.
For instance, $\varsigma_0(\langle \mathsf{Epcot} \rangle)$ is the dataset $\{\mathsf{isa}(\mathsf{Epcot},\mathsf{tp}), \mathsf{isa}(\mathsf{Epcot},\mathsf{ap})\}$ $\cup$ $\{\mathsf{located}(\mathsf{Epcot},\mathsf{Florida})\}$ $\cup$ $\{\mathsf{partOf}(\mathsf{Florida},\mathsf{US})\}$ $\cup$ $\{\top(\mathsf{Epcot})$, 
$\top(\mathsf{tp})$, 
$\top(\mathsf{ap})$,
$\top(\mathsf{Florida})$,
$\top(\mathsf{US})\}$.
The pair $\kb_0 = (D_0,\varsigma_0)$ is called selective knowledge base (SKB).

Consider the set \mbox{$\unit_0$ $=$ $\{\langle\mathsf{Discovery~Cove}\rangle$, $\langle\mathsf{Epcot}\rangle\}$;} we can refer to it as an anonymous relation or simply as a unit.
If we focus on the summaries of the tuples in $\unit_0$, then 
we can observe that formula $\bar{\varphi}_1$ $=$ $x \leftarrow \mathsf{isa}(x, \mathsf{ap})$, $\mathsf{located}(x,y)$, $\mathsf{partOf}(y,\mathsf{US})$ expresses some nexus of similarity between $\langle\mathsf{Discovery~Cove}\rangle$ and $\langle\mathsf{Epcot}\rangle$.
However, this formula  neglects at least the fact that both entities are also located in Florida according to their summaries. 
Later, it will become clear that $\bar{\varphi}_a$ in Figure~\ref{fig:queries} better and fully expresses the nexus of similarity between the two entities;
therefore, we say that $\bar{\varphi}_a$ characterizes their nexus of similarity.

The next step is to create an {\em is-a} taxonomy by classifying each entity in $\kb_0$ w.r.t. $\unit_0$, where $\bar{\varphi}_a$ is the most specific concept. 
This taxonomy, called expansion graph and 
denoted by $\mathit{eg}(\unit_0,\kb_0)$, is depicted in Figure~\ref{fig:KG2}.
Intuitively, each node $n_1$ labeled by $\varphi_1 \mapsto \unit_1$ says that $\varphi_1$ characterizes $\unit_0 \cup \{\tau\}$ for each $\tau \in \unit_1$. 
Also, if a node $n_2$ labeled by $\varphi_2 \mapsto \unit_2$ has an arc pointing to $n_1$, it means that $\varphi_1$ characterizes also $\unit_0 \cup \unit_1 \cup \unit_2$.
Thus,  we can conclude, for instance, that 
the  nexus of similarity that $\langle \mathsf{Gardaland}\rangle$ has with $\unit_0$ are higher than those that $\langle \mathsf{Leolandia} \rangle$ has with $\unit_0$,  showing that $\mathsf{Gardaland}$ is more similar to the entities of $\unit_0$ than $\mathsf{Leolandia}$.
\qed
\end{example}

We  now outline the content of the paper, emphasizing the main contributions. 
In Section~\ref{sec:KR}, we introduce our {\em nexus explanation language} and define the notion of a {\em selective knowledge base}.
Then, in Section~\ref{sec:Explanations}, we formally illustrate how to {\em characterize the nexus of similarity} between tuples of entities. 
Section~\ref{sec:expansions} introduces the notion of {\em expansion graph} and discusses how it relates to semantic similarity and entity set expansion.
In Section~\ref{sec:tasks}, we identify the {\em key reasoning tasks} associated with our framework. 
Section~\ref{sec:complexity} analyzes the {\em computational aspects} of these reasoning tasks. 
In Section~\ref{sec:related}, we review the {\em related work}.
Finally,  Section~\ref{sec:conclusion} discusses some  of our design choices and draws conclusions.

\section{Knowledge Modelling}\label{sec:KR}

\subsection{Basics on Relational Structures}

We are given three pairwise disjoint countably infinite sets of symbols: $\PS$ of {\em predicates}, $\C$ of {\em constants}, and $\V$ of {\em variables}. The latter are of the form $x$, $y$, $z$, and variations thereof. 
Constants and variables are called {\em terms}.
Each predicate $p$ has an {\em arity} consisting of a positive integer denoted by $|p|$.
The set $\PS$ contains the special unary predicate $\top$ called {\em top}.

An {\em atom} is an expression of the form $p(t_1,...,t_n)$, where $p$ is a predicate of arity $n$ and $t_1,...,t_n$ is a {\em sequence} of terms.  A {\em structure} is any set of atoms. 
The {\em domain} of a structure $S$, denoted by $\Dom_S$ \marco{or $\Dom(S)$}, is the set of terms forming the atoms of $S$.
For example, the domain of the structure $\{p_1(\mathsf{1},x,\mathsf{1}), p_2 (\mathsf{d}_1,\mathsf{1},\mathsf{d}_2,\mathsf{c})\}$ is the set $\{\mathsf{1},\mathsf{c},\mathsf{d}_1,\mathsf{d}_2,x\}$.
A structure $S$ is said to be {\em closed under $\top$} if, for each $t \in \Dom_S$, it holds that  $\top(t) \in S$.
For example, $\{\top(\mathsf{1})\}$ and
$\{p(\mathsf{1},\mathsf{2}), \top(\mathsf{1}), \top(\mathsf{2})\}$ are closed under $\top$, whereas $\{p(\mathsf{1}), \top(\mathsf{2})\}$ is not.
A structure $S$ is {\em connected} if $|S| = 1$ or $S$ can be partitioned into two connected structures $S'$ and $S''$ s.t. \mbox{$\Dom_{S'} \cap \Dom_{S''} \neq \emptyset$}.
%
%
%
E.g., $\{p(\mathsf{1})\}$ and 
$\{p(\mathsf{1}), r(\mathsf{1},x)\}$ are connected, whereas $\{p(\mathsf{1}), r(\mathsf{2},x)\}$ is not.

Consider two structures $S$ and $S'$ together with  a set $\mathit{T}$ of terms. Hereinafter, a \mbox{{\em $\mathit{T}$-homomorphism}} from $S$ to $S'$ is a total map \mbox{$h:\Dom_S \rightarrow \Dom_{S'}$} s.t. $h(t) = t$ for each $t \in T \cap \Dom_S$ and $p(t_1,...,t_n) \in S$ implies  $p(h(t_1),...,h(t_n)) \in S'$.
If $T = \emptyset$, then the prefix ``$\mathit{T}$-''  can be omitted.
%

\subsection{Anonymous Relations}

Consider some integer $n \geq 0$. An $n$-ary {\em tuple} is an expression of the form $\langle t_1,...,t_n\rangle$, where each $t_i$ is a term.
Intuitively, a tuple can be thought as an anonymous atom or, conversely, an atom can be thought as a labeled tuple.
For a tuple $\tau$ (resp.,  sequence {\bf s}) of terms, its $i$-th term is denoted by $\tau[i]$ (resp., ${\bf s}[i]$).
For a set $T$ of terms, $T^n$ denotes the set of all the $|T|^n$ tuples of arity $n$ that can be constructed by using terms from $T$. 
The {\em domain} of a set $J$ of tuples, denoted by $\Dom_J$, is the set of terms forming the tuples of $J$. For example, the domain of the set $\{\langle \mathsf{1},x,\mathsf{1}\rangle, \langle \mathsf{d}_1,\mathsf{d}_2,\mathsf{c}\rangle\}$ of tuples is the set $\{\mathsf{1},\mathsf{c},\mathsf{d}_1,\mathsf{d}_2,x\}$.

A set $J$ of $n$-ary tuples is {\em proper} if  \mbox{$n > 0$} and its arrangement into rows produces no duplicate column.
For example, the set $\{\langle \mathsf{a},\mathsf{b},\mathsf{a}\rangle, \langle \mathsf{c},\mathsf{d},\mathsf{e}\rangle\}$ is proper whereas the set $\{\langle \mathsf{a},\mathsf{b},\mathsf{a}\rangle, \langle \mathsf{c},\mathsf{d},\mathsf{c}\rangle\}$ is not.
An $n$-ary {\em anonymous relation} is any proper finite nonempty set of $n$-ary tuples of constants.
Intuitively, it encodes some ``unit'' of knowledge that we would like to characterize and possibly expand;
accordingly, hereinafter, an anonymous relation will 
be often called {\em unit} for short.
Coming back to our running example, the set $\unit_0 = \{\langle\mathsf{Discovery\_Cove}\rangle,\langle\mathsf{Epcot}\rangle\}$ given in Example~\ref{ex:running} is a unary anonymous relation and, therefore, a unary unit too.

\vspace{-0.5mm}
\subsection{Selective Knowledge Bases}
\vspace{-0.5mm}

A {\em dataset} $D$ is a finite nonempty structure closed under $\top$ s.t. \mbox{$\Dom_D\subset \C$;} see, for instance, $\db_0$ in Example~\ref{ex:running}.
A {\em summary selector} is a computable function $\varsigma$ (informally, an algorithm) that: $(i)$ takes as input a dataset $D$ together with an \mbox{$n$-ary} tuple $\tau$ of constants from $\Dom_\db$, and $(ii)$ returns a dataset \mbox{$\varsigma(\db,\tau) \subseteq D$,} called the {\em summary} of $\tau$ (w.r.t. $D$), whose domain contains at least all the constants of $\tau$.
According to Example~\ref{ex:running}, 
$\{\mathsf{located}(\mathsf{Discovery\_Cove}$,
$\mathsf{Florida})$, 
$\top(\mathsf{Florida})$, $\top(\mathsf{Discovery\_Cove})$, $\top(\mathsf{Epcot})\}$ 
is a possible summary for $\langle\mathsf{Discovery\_Cove}\rangle$ w.r.t. $\db_0$, whereas 
both
$\{\mathsf{located}(\mathsf{Epcot}$,  $\mathsf{Florida})$, 
$\top(\mathsf{Epcot})$, $\top(\mathsf{Florida})\}$ and
$\{\mathsf{located}(\mathsf{Discovery\_Cove}$,  $\mathsf{Florida})\}$ are not.
If there is no ambiguity, we write $\varsigma(\tau)$ instead of $\varsigma(\db,\tau)$.
Intuitively, $\varsigma$ selects a portion of $D$  carrying relevant information about $\tau$, depending on some specific application scenario; see, for instance, $\varsigma_0$ in Example~\ref{ex:running}.%
\footnote{There exist specific techniques working with unary tuples~\cite{DBLP:journals/ws/LiuCGQ21} and binary tuples~\cite{DBLP:journals/semweb/Pirro19}.}
%
%
%

A {\em selective knowledge base}, SKB for short, is a pair \mbox{$\kb =(\db, \varsigma)$,} where  $\db$ is a dataset and $\varsigma$
is a summary selector; 
see, for instance, $\kb_0$ in Example~\ref{ex:running}.

\vspace{-0.5mm}
\subsection{Explanation Language}
\vspace{-0.5mm}

A {\em (conjunctive) formula} is an expression $\varphi$ of the form

\vspace{-2.5mm}

\begin{equation}\label{eq:formula}
x_1,...,x_n \leftarrow p_1({\bf t}_1),..., p_m({\bf t}_m),    
\end{equation}

\vspace{-0mm}

\noindent where $n \geq 0$ is its {\em arity},
$m  > 0$ is its {\em size} often denoted by $|\varphi|$, each ${\bf t}_i$ is a sequence of terms, each $p_i({\bf t}_i)$ is an atom, and
each $x_j$ is a variable ---called {\em free}--- occurring in some of the atoms of $\varphi$.
Hereinafter, the sequence or conjunction $p_1({\bf t}_1), ..., p_m({\bf t}_m)$ of atoms is denoted by $\mathit{conj}(\varphi)$ and the set of its atoms by $\mathit{atm}(\varphi)$.
Essentially, $\varphi$ is an expression in the equality-free primitive positive fragment of first-order logic over the vocabulary $\PS \cup \C \cup \V$~\cite{DBLP:journals/jacm/Rossman08}.

A formula is said to be {\em open} if it contains at least one free variable.
A formula $\varphi$ is {\em connected} if 
\mbox{$\mathit{atm}(\varphi)$} is connected.
An open formula $\varphi$ is 
{\em  nearly connected} if $\mathit{atm}(\varphi) \cup \{\mathsf{free}(x_1,...,x_n)\}$ is connected, where 
$\mathsf{free}(x_1,...,x_n)$ is a dummy fresh atom.
Intuitively, in the latter case, each \mbox{$\alpha \in \mathit{atm}(\varphi)$} has a (direct or indirect) connection with some atom containing at least one free variable.
Clearly, if an open formula is connected, then it is also nearly connected; however, the converse may not hold.
For example, formula 
$\bar{\varphi}_1 = x  \leftarrow   \mathsf{isa}(x, \mathsf{ap}), \mathsf{located}(x,y), \mathsf{partOf}(y,\mathsf{US})$ is (nearly) connected, whereas
$\bar{\varphi}_2 = x,y  \leftarrow   \mathsf{isa}(x, \mathsf{ap}), \mathsf{partOf}(y,\mathsf{US})$ is nearly connected but not connected.
%
Differently, \mbox{$x \leftarrow   \mathsf{isa}(x, \mathsf{ap}), \mathsf{partOf}(y,\mathsf{US})$} is not even nearly connected as neither $y$ nor $\mathsf{US}$ occurs in the unique atom containing $x$.

Consider an $n$-ary open formula
$\varphi$.
Let $\varphi'$ be the formula obtained from $\varphi$ by replacing the $j$-th free variable with $z_j$ and each non-free variable $y$ with $y_{\varphi}$.
The {\em conjunction} of two $n$-ary open formulas $\varphi_1$ and $\varphi_2$, denoted by \mbox{$\varphi_1\wedge\varphi_2$,} is the new formula $z_1,...,z_n \leftarrow  \mathit{conj}(\varphi_1'), \mathit{conj}(\varphi_2')$.
For example, if $\varphi_1$ is the formula $x_1,x_2 \leftarrow p(a,x_1), r(x_2,y)$ and
$\varphi_2$ is $y_2, y_1 \leftarrow u(y_1), v(y_2,b,y,a)$, then 
$\varphi_1 \wedge \varphi_2 = z_1,z_2 \leftarrow p(a,z_1), r(z_2,y_{\varphi_1}),
u(z_2), v(z_1,b,y_{\varphi_2},a)$.
A set $F$ of formulas is {\em closed under conjunction} if, for each $\varphi_1$ and $\varphi_2$ in $F$, also $\varphi_1\wedge\varphi_2$ is in $F$.

Nearly connected formulas are the formalism we adopt to express nexus of similarity between the tuples of a unit (i.e., properties shared by such tuples) w.r.t. their summaries.
Accordingly, \marco{since
``nexus'' literally means ``series of connections'',}
such formulas ---often called {\em (nexus) explanations} for short--- form our {\em nexus explanation language} denoted by $\WCQ$.
\marco{The rationale behind the choice of this language is discussed in Section~\ref{sec:conclusion}.} 
We conclude with a rather intuitive yet useful result of our language.

\begin{proposition}\label{prop:conj}
$\WCQ$ is closed under conjunction.
\end{proposition}

%
%


\subsection{Instances of a Formula}

Consider a formula $\varphi$ as in Equation~\ref{eq:formula}.
The {\em output} to $\varphi$ over a dataset $\db$ is the set $\varphi(\db)$ of every tuple $\langle t_1,...,t_n\rangle$ admitting a \mbox{$\C$-homomorphism} $h$  from $\mathit{atm}(\varphi)$ to $D$  s.t. each $h(x_i)= t_i$.
An {\em instance} of $\varphi$ according to some \mbox{$\kb =(\db, \varsigma)$} is any tuple $\tau \in \varphi(\varsigma(\tau))$.
%
Intuitively, since summaries list relevant information about tuples, if $\tau$ is not in the output to $\varphi$ over $\varsigma(\tau)$, then $\varphi$ does not express properties of $\tau$ in terms of the considered scenario; if so, we consider $\tau$ not an instance of $\varphi$ according to $\kb$.
The set of all \mbox{$\varphi$-instances} is denoted by $\mathit{inst}(\varphi,\kb)$.
According to Example~\ref{ex:running}, given the formula $\bar{\varphi}_3$ defined as $y  \leftarrow  \mathsf{located}(x,y), \mathsf{partOf}(y,\mathsf{US})$, 
we have that $\bar{\varphi}_3(\db_0) = \{\langle \mathsf{Florida}\rangle, \langle \mathsf{California}\rangle \}$ and
\mbox{$\mathit{inst}(\bar{\varphi}_3,\kb_0) = \emptyset$.}
%
The latter holds because if
$c \in \{\mathsf{Florida}$, $\mathsf{California}\}$, then $\varsigma_0(\langle c \rangle)$ is $\{\mathit{partOf}(c,\mathsf{US}), \top(c), \top(\mathsf{US})\}$ and, thus, $\bar{\varphi}_3(\varsigma_0(\langle c \rangle)) = \emptyset$.

\begin{proposition}
Consider some \marco{selective KB} $\kb =(\db, \varsigma)$ and some formula $\varphi$. It holds that $\mathit{inst}(\varphi,\kb) \subseteq \varphi(D)$.
\end{proposition}

We end with a result being quite
useful since $\WCQ$ is closed under conjunction.

\begin{proposition}\label{prop:instConj}
Consider a
\marco{selective KB} $\kb =(\db, \varsigma)$ and two formulas $\varphi_1$ and $\varphi_2$ of the same arity.
It holds that $\mathit{inst}(\varphi_1 \wedge \varphi_2,\kb) = \mathit{inst}(\varphi_1,\kb) \cap \mathit{inst}(\varphi_2,\kb)$.
\end{proposition}

\subsection{Relationships between Formulas}

Consider two formulas $\varphi_1$ and $\varphi_2$  s.t. $x_1,...,x_n$ and $y_1,...,y_n$ are their free variables, respectively.
Let \mbox{$h : \Dom_{\mathit{atm}(\varphi_1)} \rightarrow \Dom_{\mathit{atm}(\varphi_2)}$} be a map  s.t.
each \mbox{$h(x_i) =  y_i$.}
If $h$ is a \mbox{$\C$-homomorphism} from 
$\mathit{atm}(\varphi_1)$ to $\mathit{atm}(\varphi_2)$, then this fact is denoted by \mbox{$\varphi_1 \longrightarrow \varphi_2$.}
If \mbox{$\varphi_1 \longrightarrow \varphi_2$} and \mbox{$\varphi_2 \longrightarrow \varphi_1$} both hold, 
then the two formulas are said {\em equivalent}, denoted by \mbox{$\varphi_1 \longleftrightarrow \varphi_2$}.
Moreover, if $h$ is also a bijection, then the formulas are said {\em isomorphic}, denoted by \mbox{$\varphi_1 \simeq \varphi_2$,} implying they are also equivalent. 
\marco{Given a formula $\varphi$, as common, $[\varphi]$ denotes the {\em equivalence class} of $\varphi$ collecting all formulas $\varphi'$  s.t. $\varphi \simeq \varphi'$.}

\begin{proposition}\label{prop:equinst}
If \mbox{$\varphi_1 \longleftrightarrow \varphi_2$}, then \mbox{$\varphi_1(D) = \varphi_2(D)$} for each  $D$, and \mbox{$\mathit{inst}(\varphi_1,\kb) = \mathit{inst}(\varphi_2,\kb)$}
for each \marco{SKB} $\kb$.
\end{proposition}

\marco{It is well-known that  every formula $\varphi$ admits ---up to isomorphism--- a unique equivalent formula of minimal size, denoted by $\mathit{core}(\varphi)$, called the {\em core} of $\varphi$~\cite{DBLP:journals/tods/FaginKP05}.
Thus, \mbox{$\varphi_1 \longleftrightarrow \varphi_2$} if, and only if,
$[\mathit{core}(\varphi_1)] = [\mathit{core}(\varphi_2)]$.
For example, from the formulas $\bar{\varphi}_4 = x   \leftarrow  
\mathsf{isa}(x, \mathsf{tp})$, $\mathsf{located}(x,y)$,
$\mathsf{located}(x,\mathsf{Florida})$ and
$\bar{\varphi}_5 = z  \leftarrow   
\mathsf{isa}(z, \mathsf{tp})$,  $\mathsf{located}(z,\mathsf{Florida})$ we get that $\bar{\varphi}_4 \longleftrightarrow \bar{\varphi}_5$ and that $[\mathit{core}(\bar{\varphi}_5)] = [\mathit{core}(\bar{\varphi}_4)] = [\bar{\varphi}_5$].}

\section{Nexus Characterizations}\label{sec:Explanations}

In the rest of this paper, consider to have a selective KB \mbox{$\kb =(\db, \varsigma)$} together with an $n$-ary unit $\unit = \{\tau_1,...,\tau_m\}$  such that $\Dom_\unit \subseteq \Dom_{\db}$, and refer $\unit$ as a {\em $\kb$-unit} for short. See, for instance, the $\kb_0$-unit $\unit_0$ in Example~\ref{ex:running}.
%


\begin{definition}\label{def:interpretations}
A  formula $\varphi$ {\em interprets (some)
nexus of similarity} between the tuples of the given $\kb$-unit $\unit$ 
if both the next hold: $(i)$ $\varphi \in \WCQ$; and $(ii)$ $\mathit{inst}(\varphi,\kb) \supseteq \unit$. 
\end{definition}

When the Definition~\ref{def:interpretations} applies, we may also say that $\varphi$ interprets $\unit$, that 
$\varphi$ is an interpretation for $\unit$, or that 
$\unit$ is interpreted by $\varphi$ (w.r.t. $\kb$).
For instance, according to Example~\ref{ex:running}, formula $\bar{\varphi}_1$ $=$ $x  \leftarrow   \mathsf{isa}(x, \mathsf{ap})$, $\mathsf{located}(x,y)$, $\mathsf{partOf}(y,\mathsf{US})$ interprets $\unit_0$.
Since datasets are closed under $\top$, 
the explanation ${x_1,...,x_n}$ $\leftarrow$  $\top(x_1),\!...,\!\top(x_n)$ always interprets any $n$-ary unit. Hence, the next result follows.

\begin{proposition}\label{prop:interpreted}
The given $\kb$-unit $\unit$ is always interpreted by some explanation.
\end{proposition}

One may note, however, that formula  $\bar{\varphi}_1$ above  is not very precise. Intuitively, although
\mbox{$\mathit{inst}(\bar{\varphi}_1,\kb_0) \supseteq \unit_0$} holds,
$\bar{\varphi}_1$ neglects that the entities in $\unit_0$ are ``located in Florida''. 
In a sense, $\bar{\varphi}_1$ only partially expresses the nexus of similarity between the tuples of $\unit_0$. 
As an effect, $\langle\mathsf{Pacific\_ Park}\rangle$ ---which is not in Florida--- also belongs to $\mathit{inst}(\bar{\varphi}_1,\kb_0)$. Thus, we give the following definition.
%

\begin{definition}\label{def:characterizes}
\marco{A  formula $\varphi$ {\em characterizes (all) the nexus of similarity} between the tuples of the given $\kb$-unit $\unit$ if both the next conditions hold: $(i)$ $\varphi$ interprets $\unit$; and $(ii)$ \mbox{$\varphi \longrightarrow \varphi'$} implies $\varphi' \longrightarrow \varphi$  whenever $\varphi'$ interprets $\unit$.} 
\end{definition}

When the Definition~\ref{def:characterizes} applies, we may also say, for short, that $\varphi$ characterizes $\unit$, that 
$\varphi$ is a characterization for $\unit$, or that 
$\unit$ is characterized by $\varphi$ (w.r.t. $\kb$).
%
Considering, for example, the formula $\bar{\varphi}_a$ given in Figure~\ref{fig:queries}, we get that $\bar{\varphi}_1$
does not characterize $\unit_0$ since $\bar{\varphi}_1 \longrightarrow \bar{\varphi}_a$ holds, $\bar{\varphi}_a$ also interprets $\unit_0$, but $\bar{\varphi}_a \longarrownot\longrightarrow \bar{\varphi}_1$. 
Indeed, $\bar{\varphi}_a$ does characterize $\unit_0$ and the reason will be clear at the end of this section. 
From Definition~\ref{def:characterizes}, if $\varphi_1$ and $\varphi_2$ characterize $\unit$, then they could be either equivalent ($\varphi_1 \longrightarrow \varphi_2$ and $\varphi_2 \longrightarrow \varphi_1$) or incomparable ($\varphi_1 \longarrownot\longrightarrow \varphi_2$ and $\varphi_2 \longarrownot\longrightarrow \varphi_1$). 
Since, by Proposition~\ref{prop:conj}, $\WCQ$ is closed under conjunction, the next results follow.

\begin{proposition}\label{pro:equiva}
If two explanations characterize $\unit$, then they are equivalent.
\end{proposition}


\begin{cor}\label{cor:char} A  formula $\varphi$ {\em characterizes} the given $\kb$-unit $\unit$ if, and only if, both the following hold: $(i)$ $\varphi$ interprets $\unit$; and $(ii)$ $\varphi' \longrightarrow \varphi$  whenever $\varphi'$ interprets $\unit$.
\end{cor}

Indeed, in a sense, a formula characterizing $\unit$ fully expresses it.
What remains to know is whether every unit admits such a formula. 
The next proposition provides an affirmative answer to this question.

\begin{theorem}\label{thm:characterized}
The given $\kb$-unit $\unit$ is always characterized by some explanation.
\end{theorem}

%

The remaining of the section is devoted to prove this  important result. 
Hereafter, for $i,j \in\mathbb{N}^+$, if $i \leq j$, then $[i..j]$ is a shorthand for $\{i,..,j\} \subseteq \mathbb{N}^+$, else $[i..j]$ is the empty set.
In particular, if $i = 1$, then $[j]$ is a short form for $[1..j]$.
Consider the $n$-ary tuples 
$\bar{\tau}_1, ..., \bar{\tau}_\ell$.
Their {\em direct product}, denoted by $\bar{\tau}_1  \otimes ... \otimes \bar{\tau}_\ell$, is the sequence $d_{\bar{\bf s}_1},...,d_{\bar{\bf s}_n}$ of constants, where each $\bar{\bf s}_i$ is  
$\bar{\tau}_1[i],...,\bar{\tau}_\ell[i]$.
For example, the direct product of the tuples
$\langle 1,2\rangle$, $\langle 3,4\rangle$, $\langle 5,6\rangle$ is the sequence $d_{1,3,5}, d_{2,4,6}$.
Accordingly, given the datasets $D_1, ..., D_k$, their {\em direct product}, denoted by $D_1 \otimes...\otimes D_k$, is the structure 
\[
\left\{p(\langle{\bf c}_1\rangle \otimes ... \otimes \langle{\bf c}_k\rangle)~:~p({\bf c}_1) \in 
D_1 , ... ,  p({\bf c}_k)\in D_k\right\}.
\]

\begin{proposition}
The direct product of datasets is a dataset.
\end{proposition}

Our goal is to build a {\em canonical characterization} of the \mbox{$n$-ary} unit $\unit\!=\!\{\tau_1,...,\tau_m\}$ according to $\kb\!=\!(\db, \varsigma)$, called $\mathit{can}(\unit,\kb)$.
First, we build the set $\FreeConst = \{d_{{\bf s}_1},...,d_{{\bf s}_n}\}$ of constants used to determine the free variables of $\mathit{can}(\unit,\kb)$, where 
$d_{{\bf s}_1},...,d_{{\bf s}_n}$ $=$ $\tau_1 \otimes...\otimes \tau_m$.
Then, we build the dataset $P = \varsigma(\tau_1) \otimes ... \otimes \varsigma(\tau_m)$ used to determine some atoms of $\mathit{can}(\unit,\kb)$.
%
%
Consider now some $d_{\bf s}$ occurring in $P$, and let $\Dom_{\bf s}$ be the set of constants of ${\bf s}$.
If $|\Dom_{\bf s}| = 1$, then the atoms of $P$ containing $d_{\bf s}$
might have to be ``cloned'' to determine some extra 
atoms of $\mathit{can}(\unit,\kb)$.
Let $\Genes = \{d_{\bf s} \in \Dom_P : |\Dom_{\bf s}| = 1\}$ be the set of constants used as possible ``genes'' for such clones.
For each $d_{\bf s} \in \Dom_P$, let

\vspace{-4.5mm}

\begin{equation*}
	f(d_{\bf s})\ =\ \left\{
	\begin{array}{ll}
		\{ d_{\bf s}, {\bf s}[1] \} & \ \ \ \text{if }  d_{\bf s} \in \FreeConst \, \wedge \, d_{\bf s} \in \Genes\\
  
		\{ d_{\bf s} \} & \ \ \ \text{otherwise}.
	\end{array} \right.
\end{equation*}

\vspace{-2mm}

\noindent Accordingly, for any atom $\alpha = p(t_1,...,t_k)$ of $P$, we define: $\mathit{clones}(\alpha)$ as the set $\{ p(c_1,...,c_k):\mbox{each}~c_i \in f(t_i) \} \setminus \{\alpha\}$
and $C = \{ \alpha' \in\mathit{clones}(\alpha): \alpha \in P\}$.
\marco{Let $\mu$ be the mapping $\{c \mapsto c: c \in \Dom_C \setminus \Dom_P\}$ $\cup$ $\{ d_{\bf s} \mapsto g(d_{\bf s}) : d_{\bf s} \in \Dom_P\}$ used to transform atoms of $P \cup C$ into atoms of  $\mathit{can}(\unit,\kb)$, where }

\vspace{-2.5mm}

\begin{equation*}
	g(d_{\bf s}) \ =\ \left\{
	\begin{array}{ll}
		x_{\bf s} & \text{if } d_{\bf s} \in \FreeConst \\
		y_{\bf s}  & \text{if } d_{\bf s} \not\in \FreeConst \, \wedge \, d_{\bf s} \not\in \Genes \\
		{\bf s}[1]  & \text{if } d_{\bf s} \not\in \FreeConst \, \wedge \, d_{\bf s} \in \Genes .
	\end{array} \right.
\end{equation*}

\vspace{-0.5mm}

\noindent Consider the next formula (``$\wedge$'' is used instead of ``$,$''):

\vspace{-4.5mm}

\begin{equation*}
	\Phi =x_{{\bf s}_1},...,x_{{\bf s}_n} \leftarrow  \bigwedge_{p(t_1,...t_k) \in P \cup C} \ p(\mu(t_1),...,\mu(t_k)).
\end{equation*}

\vspace{-2mm}

\noindent Let $\mathit{can}(\unit,\kb)$ be the formula obtained from $\Phi$ by removing atoms that are not in the maximal (under subset inclusion) nearly connected part of $\mathit{atm}(\Phi)$.

\begin{example}
Consider the \marco{SKB} $\kb$ $=$ $(\db, \varsigma)$ and the unit $\unit$ $=$ 
$\{\langle 1,1 \rangle, \langle 1,2 \rangle\}$, 
where $\db$ $=$ $\{r(1,2)$, $r(2,1)$, $s(2,1)$, $s(1,2)$, $\top(1)$, $\top(2)\}$, and $\varsigma$ is s.t. $\varsigma(\langle 1,1\rangle)$ $=$ $\{r(1,2)$, $s(1,2)$, $\top(1)$, $\top(2)\}$ and 
$\varsigma(\langle 1,2\rangle)$ $=$ $\{r(1,2)$, $s(2,1)$, $\top(1)$, $\top(2)\}$.
By computing the direct product between the tuples of $\unit$ we get
$\mathit{Fr}$ $=$ $\{d_{1,1},d_{1,2}\}$.
By computing the direct product between the summaries of the tuples of $\unit$ we get
$P$ $=$ 
$\{r(d_{1,1},d_{2,2})$, 
$s(d_{1,2},d_{2,1})$,
$\top(d_{1,1})$,
$\top(d_{1,2})$,
$\top(d_{2,1})$,
$\top(d_{2,2})\}$;
thus, the domain of $P$ is $\Dom_P$ $=$ 
$\{d_{1,1}$, 
$d_{1,2}$, 
$d_{2,2}$, 
$d_{2,1}\}$.
We can now compute the set $\mathit{Ge}$ $=$ $\{d_{1,1}$, $d_{2,2}\}$ of genes.
Accordingly, $f(d_{1,1})$ $=$ $\{d_{1,1},1\}$,
$f(d_{1,2})$ $=$ $\{d_{1,2}\}$,
$f(d_{2,1})$ $=$ $\{d_{2,1}\}$, and
$f(d_{2,2})$ $=$ $\{d_{2,2}\}$.
Hence, $C$ $=$ $\{r(1,d_{2,2})$, $\top(1)\}$.
Also, 
$g(d_{1,1})$ $=$ $x_{1,1}$, 
$g(d_{1,2})$ $=$ $x_{1,2}$, 
$g(d_{2,1})$ $=$ $y_{2,1}$, and 
$g(d_{2,2})$ $=$ $2$.
Finally,
$\mu$ $=$
$\{1\!\mapsto\!1\}$ $\cup$ 
$\{d_{1,1}\!\mapsto\!x_{1,1}$,
$d_{1,2}\!\mapsto\!x_{1,2}$,
$d_{2,1}\!\mapsto\!y_{2,1}$,
$d_{2,2}\!\mapsto\!2\}$ so that
$\Phi$ $=$ $\mathit{can}(\unit,\kb)$ $=$ $x_{1,1},x_{1,2}$ 
$\leftarrow$ 
$r(x_{1,1},2)$, 
$s(x_{1,2},y_{2,1})$,
$\top(x_{1,1})$,
$\top(x_{1,2})$,
$\top(y_{2,1})$,
$\top(2)$,
$r(1,2)$, 
$\top(1)$.
Clearly, $\langle 1,1 \rangle$ is in the output to $\mathit{can}(\unit,\kb)$ over $\varsigma(\langle 1,1\rangle)$ and
 $\langle 1,2 \rangle$ is in the output to $\mathit{can}(\unit,\kb)$ over $\varsigma(\langle 1,2\rangle)$.
\hfill {\tiny $\blacksquare$}
\end{example}

\begin{proposition}\label{thm:can}
It holds that $\mathit{can}(\unit,\kb)$ characterizes $\unit$.
\end{proposition}

%
Clearly, Theorem~\ref{thm:characterized} now follows as a corollary of Proposition~\ref{thm:can}.
Concerning our running example,  $\mathit{can}(\unit_0, \kb_0)$ is isomorphic to $x \leftarrow \mathit{conj}(\bar{\varphi}_a)$, $\mathsf{isa}(x,y_1)$, $\mathsf{isa}(x,y_2)$, $\top(y_1)$, $\top(y_2)$. 
One can verify that \mbox{$\bar{\varphi}_a \longleftrightarrow \mathit{can}(\unit_0, \kb_0)$} and, thus, 
we can conclude that $\bar{\varphi}_a$ also characterizes $\unit_0$.
Since $\mathit{can}(\unit,\kb)$ may contain ``redundant'' atoms (e.g., the atoms after $\mathit{conj}(\bar{\varphi}_a)$ above), one can consider its core and talk about  the {\em core characterization} of $\unit$ according to $\kb$, denoted by $\mathit{core}(\unit,\kb)$.
Indeed, \mbox{$\bar{\varphi}_a \simeq \mathit{core}(\unit_0,\kb_0)$.}

\vspace{-1mm}
\section{The Expansion Graph}\label{sec:expansions}
\vspace{-0.5mm}

\marco{Let $\mathbf{T}$ be the set of all $n$-ary tuples over $\Dom_D$.
Our ultimate goal is to classify every tuple in $\mathbf{T}$ w.r.t. $\unit$.
Intuitively, this means building an {\em is-a} taxonomy with multiple inheritance having exactly one most specific concept consisting of a characterization of $\unit$ and where any other concept generalizing $\unit$ is a characterization of $\unit \cup \{\tau\}$ for some  $\tau \in \mathbf{T}$.}


\begin{definition}\label{def:def}
Consider a formula $\varphi$ that 
interprets the given $\kb$-unit $\unit$.
If $\mathit{inst}(\varphi,\kb) = \unit$, then $\varphi$ {\em defines} $\unit$.
\end{definition}

When the above definition applies, we may also say, for short, that $\unit$ is {\em $\kb$-definable}.
The next proposition makes a bridge between characterizations and definability.

\begin{proposition}\label{prop:def-car}
The given unit $\unit$ is $\kb$-definable if, and only if, any explanation that characterizes $\unit$ also defines it.
\end{proposition}

In what follows, the prefix ``$\kb$-'' is omitted if clear from the context.
Coming back to our running example, according to Definitions~\ref{def:characterizes}~and~\ref{def:def}, we know that $\bar{\varphi}_a$ also defines $\unit_0$ beyond characterizing it.
Moreover, we also know that \mbox{$\unit_0' = \unit_0 \, \cup  \, \{\langle\mathsf{Florida}\rangle\}$} is not definable; indeed, regardless the way how summaries are constructed, 
\marco{$\bar{\varphi}_f$ characterizes both $\unit_0'$ and the unit collecting all the entities of $\kb_0$.}
%
Finally, we know that $\bar{\varphi}_5 = z  \leftarrow   
\mathsf{isa}(z, \mathsf{tp})$,  $\mathsf{located}(z,\mathsf{Florida})$ does not characterize $\unit_0$ although defining it; indeed, $\bar{\varphi}_5 \longrightarrow \bar{\varphi}_a$ but $\bar{\varphi}_a \longarrownot\longrightarrow \bar{\varphi}_5$.
Thus, defining and characterizing a unit are rather different tasks. 
This apparent drawback, however, 
can be finely overcome via the next two important notions.

First, we define the {\em expansions} of $\unit$, denoted by $\mathit{ex}(\unit,\kb)$, as the set of every definable unit $\unit' \supseteq \unit$.

\begin{proposition}\label{prop:dex}
$\mathit{ex}(\unit,\kb)$ is finite and never empty.
\end{proposition}


\noindent Then, we define the {\em essential expansion} of $\unit$, denoted by $\mathit{ess}(\unit,\kb)$, as the intersection of all the units in $\mathit{ex}(\unit,\kb)$.

\begin{proposition}\label{thm:ess}
    It holds that $\mathit{ess}(\unit,\kb) \in \mathit{ex}(\unit,\kb)$.
\end{proposition}

Clearly, $\mathit{ess}(\unit,\kb)$ is the smallest definable unit containing all the tuples of the $\kb$-unit $\unit$:
there is no formula characterizing $\unit$ that does not 
characterize $\mathit{ess}(\unit,\kb)$.
%
%
More in general, a formula that characterizes $\unit$ also characterizes each unit $\unit'$  s.t. $\unit \subseteq \unit' \subseteq \mathit{ess}(\unit,\kb)$.



%

\begin{definition}
\marco{The {\em expansion graph} of $\unit$ according to $\kb$, denoted $\mathit{eg}(\unit,\kb)$, is the triple $(N,A,\delta)$ where:
$(i)$ $N$ is the set $\{[\mathit{core}(\unit \cup \{\tau\},\kb)] : \tau \in \mathbf{T}\}$;
$(ii)$ $A$ collects each pair $([\varphi_1],[\varphi_2])$ of distinct nodes of $N$, where $\varphi_2 \longrightarrow \varphi_1$ and there is no $[\varphi_3] \in N \setminus \{[\varphi_1],[\varphi_2]\}$ such that $\varphi_2 \longrightarrow \varphi_3$ and $\varphi_3 \longrightarrow \varphi_1$; and
$(iii)$ $\delta$ maps every $[\varphi] \in N$ to the set $\mathit{inst}(\varphi,\kb) \setminus \{\tau \in \mathit{inst}(\varphi',\kb):([\varphi'],[\varphi]) \in A\}$, named the set of all {\em direct instances} of each formula in $[\varphi]$.}
\end{definition}

Since $\varphi_2\in[\varphi_1]$ implies that $\varphi_1\longleftrightarrow \varphi_2$, by Proposition~\ref{prop:equinst} we get the following result.

\begin{proposition}
Function $\delta$ in $\mathit{eg}(\unit,\kb)$ is well-defined.
\end{proposition}

%
The expansion graph $\mathit{eg}(\unit_0,\kb_0)$ of our running example is shown in Figure~\ref{fig:KG2}; node-labels are of the form \mbox{$\varphi \mapsto \delta([\varphi])$,} where $\varphi$ is one representative element of $[\varphi]$; the shapes of the core  characterizations are given in \mbox{Figure~\ref{fig:queries}.}
Since $\mathit{inst}(\bar{\varphi}_a,\kb_0)\! =\! \unit_0$, we have that $\unit_0$ is definable and, thus,
the direct instances of $\bar{\varphi}_a$ are those in $\unit_0$.

\begin{theorem}\marco{
$\mathit{eg}(\unit,\kb)$ is a directed acyclic graph; 
the sets of direct instances form a partition of $\mathbf{T}$;
$[\mathit{core}(\unit,\kb)]$ is the only source node; and $\delta([\mathit{core}(\unit,\kb)]) = \mathit{ess}(\unit,\kb)$.}
\end{theorem}

\begin{proof}[Proof sketch]
$(1)$ Graph $\mathit{eg}(\unit,\kb)$ is directed by construction; 
%
%
%
%
to see its acyclicity, note that the presence of a cycle would imply that all involved  nodes would be the same equivalence class, which is not possible.
%
$(2)$ Let $\tau \in \mathbf{T}$. By construction, there is $[\varphi] \in N$ s.t.  $[\varphi] = [\mathit{core}(\unit \cup \{\tau\},\kb)]$. Hence, $\varphi$ characterizes $\unit \cup \{\tau\}$ and $\tau \in \mathit{inst}(\varphi,\kb)$.
First we claim that $\tau \in \delta([\varphi])$. By contradiction, there is some  $([\varphi'], [\varphi])\in A$ s.t. $\varphi \longrightarrow \varphi'$, $\varphi' \longarrownot\longrightarrow \varphi$, and $\tau \in \mathit{inst}(\varphi',\kb)$. 
Thus,  $\varphi'$ interprets  $\unit \cup \{\tau\}$.
But $\varphi' \longarrownot\longrightarrow \varphi$ violates Corollary~\ref{cor:char}.
Assume now there is $[\varphi'] \in N$ different from $[\varphi]$  s.t. $\tau \in \delta([\varphi'])$.
Hence, $\tau \in \mathit{inst}(\varphi',\kb)$ and, thus,  $\varphi'$ interprets $\unit \cup \{\tau\}$. By Corollary~\ref{cor:char},  
$\varphi' \longrightarrow \varphi$.
Thus, there is a path from $[\varphi]$ to $[\varphi']$.
This means that  $\tau \not\in \delta([\varphi])$, which is a contradiction.
%
%
$(3)$ To show that $[\mathit{core}(\unit,\kb)]$ is the only source, note that such a node is equal to $[\mathit{core}(\unit \cup \{\tau\},\kb)]$ (being in $N$ by definition), for each $\tau \in \unit$; its uniqueness derives from Corollary~\ref{cor:char} since, for each $\tau \in {\bf T}$, formula  $\mathit{core}(\unit\cup\{\tau\},\kb)$  interprets $\unit$ while $\mathit{core}(\unit,\kb)$ characterizes $\unit$.
$(4)$ Finally, note that $\delta([\mathit{core}(\unit,\kb)]) = \mathit{inst}([\mathit{core}(\unit,\kb)],\kb)$ since $[\mathit{core}(\unit,\kb)]$ is a source; thus,
$\mathit{inst}([\mathit{core}(\unit,\kb)],\kb) = \mathit{ess}(\unit,\kb)$ by Propositions~\ref{prop:def-car} and~\ref{thm:ess}.
%
%
\end{proof}

Expansion graphs provide a natural bridge from
nexus characterizations to semantic similarity and set expansion.
First, one may see that the  nexus of similarity that $\langle \mathsf{Gardaland}\rangle$ has with $\unit_0$ are higher than those that $\langle \mathsf{Leolandia} \rangle$ has with $\unit_0$; also
the  nexus of similarity that $\langle \mathsf{Prater}\rangle$ has with $\unit_0$ coincide with those  that $\langle \mathsf{Leolandia}\rangle$ has with $\unit_0$.
%
%
Second, existing techniques designed to work over taxonomies~\cite{DBLP:journals/csur/ChandrasekaranM21} can be used with an expansion graph to get also similarity scores.
Third, the navigation of $\mathit{eg}(\unit,\kb)$  from its source node \marco{$[\mathit{core}(\unit,\kb)]$} provides possible ``expansion plans'' of the initial unit $\unit$;
for example, one may note that $\unit_0 = \mathit{ess}(\unit_0,\kb_0)$ has three possible expansion plans in $\mathit{eg}(\unit_0,\kb_0)$, depending on whether $\langle \mathsf{Gardaland} \rangle$ and $\langle \mathsf{Pacific\_Park} \rangle$ are given one after the other or together.
Forth, to improve the development of robust similarity benchmarks, characterizations and expansions graphs may help human experts to have a uniform view on the nexus of similarity between the entities under analysis.

\section{Key Reasoning Tasks}\label{sec:tasks}

%
Let $\omega$ denote the maximum arity inside $D$.
\marco{Given two $n$-ary tuples $\tau$ and $\tau'$ over $\Dom_D$, we say that: 
$(i)$ the  nexus of similarity that $\tau$ has with $\unit$ are {\em higher than} those that $\tau'$ has with $\unit$, 
written \mbox{$\tau \! \prec_{^\unit}^{_\kb} \! \tau'$,} if $\mathit{ess}(\unit \cup \{\tau\},\kb) \subset \mathit{ess}(\unit \cup \{\tau'\},\kb)$;
$(ii)$ the  nexus of similarity that $\tau$ has with $\unit$  {\em coincide with} those  that $\tau'$ has with $\unit$ (w.r.t. $\kb$), written \mbox{$\tau \! \sim_{^\unit}^{_\kb} \! \tau'$,} if
$\mathit{ess}(\unit \cup \{\tau\},\kb)\!=\!\mathit{ess}(\unit \cup \{\tau'\},\kb)$; and
$(iii)$ the  nexus of similarity that $\tau$ has with $\unit$ are {\em incomparable to} those that $\tau'$ has with $\unit$, 
(w.r.t. $\kb$), written \mbox{$\tau \|_{^\unit}^{_\kb} \tau'$,} if 
	$\mathit{ess}(\unit \cup \{\tau\},\kb)$ and $\mathit{ess}(\unit \cup \{\tau'\},\kb)$ are incomparable.}

\begin{table}[t!]
	\centering\footnotesize
	\renewcommand{\arraystretch}{1.1}
	\begin{tabular}{ccc}
		\hline
		{\bf  Problem} & {\bf Input} & {\bf Reasoning Task}   \\
		\hline
		\textsc{can}  & $\kb, \unit$  & compute $\mathit{can}(\unit,\kb)$ \\
		\textsc{core}  & $\kb, \unit$  & compute $\mathit{core}(\unit,\kb)$\\
        
		\textsc{def}  & $\kb, \unit$  & is $\unit$ $\kb$-definable? \\
		\textsc{ess} & $\kb, \unit, \tau$ & does $\tau \in \mathit{ess}(\unit,\kb)$ hold?\\
  \hline
		\textsc{prec} & $\kb, \unit, \tau, \tau'$ & does $\tau \! \prec_{^\unit}^{_\kb} \! \tau'$ hold?\\
		\textsc{sim} & $\kb, \unit, \tau, \tau'$ & does $\tau \! \sim_{^\unit}^{_\kb} \! \tau'$ hold?\\
		\textsc{inc} & $\kb, \unit, \tau, \tau'$ & does $\tau \|_{^\unit}^{_\kb} \tau'$ hold?\\
		\hline
\textsc{gad{\scriptsize1}}  & $\kb, \unit, \tau, \tau'$  & does $\tau \in \mathit{ess}(\unit \cup \{\tau'\},\kb)$ hold? \\
		\textsc{gad{\scriptsize2}} & $\kb, \unit, \tau, \tau'$ & does $\tau' \in \mathit{ess}(\unit \cup  \{\tau\},\kb)$ hold? \\		
		\hline
	\end{tabular}\normalsize\vspace{-2mm}\caption{Key reasoning tasks; everywhere $\{\tau,\tau'\} \cap \unit = \emptyset$.}\vspace{-2mm}\label{tab:tasks}
\end{table}
\renewcommand{\arraystretch}{1}

We are now ready to define the computational problems reported in Table~\ref{tab:tasks}, where we specify their inputs and the associated reasoning tasks. 
Problems \textsc{can},  \textsc{core},  \textsc{def} and  \textsc{ess} allow to characterize and define units; \textsc{prec},  \textsc{sim} and \textsc{inc} are allow to compare tuples within an expansion graph; \textsc{gad{\scriptsize1}} and \textsc{gad{\scriptsize2}} are useful gadgets to simplify the analysis for some of the previous problems.
For each problem $\pi$ and each $k>0$, by $\pi_k$ we denote  $\pi$ on input units of arity $k$.

Reasoning tasks may be performed under three different practical assumptions:
$\mathit{general}$, where the maximum arity $\omega$ is bounded by some integer constant (indeed, this value is generally ``small'' and it is two if $\kb$ is built from a KG) and $\varsigma$ is polynomial-time computable;
$\mathit{medial}$, where, in addition, \mbox{$m = |\unit| $} is bounded (this value is typically two in semantic similarity, and typically referred to as ``small'' in entity set expansion and other analogous tasks); and 
$\mathit{agile}$, where, in addition, the size of summaries is  bounded (this value is typically referred to as ``small'' in entity summarization).

\section{Computational Analysis}\label{sec:complexity}

\marco{This section analyzes the nine problems reported in Table~\ref{tab:tasks}. In general, we study both
upper bounds (UBs) and lower bounds (LBs).
For the gadgets, we only study UBs.}

\subsubsection{Elements of Complexity Theory.}

For decision tasks, we consider the classes
$\mathsf{NL}$ $\subseteq$ $\mathsf{P}$ $\subseteq$  $\mathsf{NP}$, $\mathsf{coNP}$ $\subseteq$ $\mathsf{DP}$ $\subseteq$ $\mathsf{PH}$ $\subseteq$ $\mathsf{PSPACE}$ $=$
$\mathsf{NPSPACE}$ $\subseteq$ $\mathsf{EXP}$ $\subseteq$ $\mathsf{NEXP}$, $\mathsf{coNEXP}$ $\subseteq$ $\mathsf{DEXP}$ where, $\mathsf{DP}$ (resp., $\mathsf{DEXP}$) is the closure of $\mathsf{NP} \cup \mathsf{coNP}$ (resp., $\mathsf{NEXP} \cup \mathsf{coNEXP}$) under intersection. For functional tasks, we consider 
$F\mathsf{P}$ $=$ $F\mathsf{P}^{\mathsf{NL}}$ $\subseteq$ $F\mathsf{P}^{\mathsf{NP}}$ $\subseteq$ $F\mathsf{P}^{\mathsf{PH}}$ $\subset$ $F\mathsf{PSPACE}$ $=$  $F\mathsf{PSPACE}^{\mathsf{NPSPACE}}$ $\subseteq$ $F\mathsf{EXP}^{\mathsf{NP}}$
where, in particular, $F\mathsf{PSPACE}$ are the  problems solvable by a Turing Machine with a one-way, write-only unlimited output tape and constantly many work tapes of polynomial length.

\subsubsection{Canonical vs Core Characterizations.}

By structurally analyzing $\mathit{can}(\unit,\kb)$, we note that  it is exponential in the general case (despite $\omega$ is bounded), polynomial in the medial case (now $m$ is also bounded) and constant in the agile case (as also summaries are bounded). Indeed,

\begin{theorem}\label{thm:canBound}
Let $c$ be the constant $2^\omega$. It holds that

\vspace{-4mm}

\[
|\mathit{core}(\unit,\kb)|\leq |\mathit{can}(\unit,\kb)| \leq  c \, \cdot \!\! \prod_{i\in [m]}  |\varsigma(\tau_i)|   \leq  c \cdot |D|^m.
\]

\vspace{-1mm}

\noindent In particular, there exists a family $\{(\unit_{\bar{m}},\kb_{\bar{m}})\}_{{\bar{m}}>0}$, where $\unit_{\bar{m}} = \{\bar{\tau}_1,...,\bar{\tau}_{\bar{m}}\}$ is a unary unit and  $\kb_{\bar{m}} = (D_{\bar{m}}, \bar{\varsigma})$ 
%
%
is \marco{an SKB} with $\bar{\varsigma}(D_{\bar{m}},\bar{\tau}_m) \subset \bar{\varsigma}(D_{\bar{m}+1},\bar{\tau}_{{m}+1})$, such that

\vspace{-4mm}

\[
|\mathit{core}(\unit_{\bar{m}},\kb_{\bar{m}})| =  |\mathit{can}(\unit_{\bar{m}},\kb_{\bar{m}})| = 2^{1-\bar{m}} \, \cdot \!\! \prod_{i\in [{\bar{m}}]}  |\bar{\varsigma}(\bar{\tau}_i)|.
\]
\vspace{-3mm}

\end{theorem}

\newcommand{\pr}{\ensuremath{\mathsf{p}}}

\vspace{-4mm}

\begin{proof}[Proof sketch]
(UB) Holds by construction of $\mathit{can}(\unit,\kb)$. (LB) Let $\pr_{\bar{m}}$ be the \mbox{$\bar{m}$-th} prime number
and $\Gamma_{i}$ be $\{\mathsf{r}(\mathsf{c}_1^i,\mathsf{c}_2^i),...,\mathsf{r}(\mathsf{c}_{\pr_{i}-1}^i,\mathsf{c}_{\pr_{i}}^i),\mathsf{r}(\mathsf{c}_{\pr_i}^i,\mathsf{c}_{1}^i)\}$ $\cup$ $\{\top(\mathsf{c}_j^i)\!:\!j\in [\pr_i]\}$, for $i\!\in\!\mathbb{N}$. Define $D_{\bar{m}}$ $=$ $\cup_{i \in [\bar{m}]} \Gamma_{i}$ and $\bar{\varsigma}$ s.t.
$\bar{\varsigma}(D_{\bar{m}},\bar{\tau}_i)\!=\! \Gamma_{i}$ with $\bar{\tau}_i\!=\! \langle \mathsf{c}_1^i\rangle$, for  $i\!\in\![\bar{m}]$.
Let $\rho$ be a map replacing $\mathsf{c}$ with $x$.
Thus, $\mathit{core}(\unit_{\bar{m}},\kb_{\bar{m}})$ $=$  $\mathit{can}(\unit_{\bar{m}},\kb_{\bar{m}})$ is the explanation $x_1^\ell \leftarrow \rho(\Gamma_{\ell})$ where $\ell = \pr_1 \cdot...\cdot \pr_{\bar{m}}$.
\end{proof}

Now we know that $|\mathit{can}(\unit,\kb)|$ and  $|\mathit{core}(\unit,\kb)|$ asymptotically coincide. In what follows, we show that computing the latter is harder. 
We start by Algorithm~\ref{alg:connectedCan} for computing $\mathit{can}(\unit,\kb)$.
According to the construction in Section~\ref{sec:Explanations}: line 1 constructs the set $\FreeConst$; line 2 prints the free variables of $\mathit{can}(\unit,\kb)$ and $n$ atoms that always belong to it; lines 3--4 build $P \cup C$; line 5 checks, for each $\beta \in P \cup C$, whether $\mu(\beta)$ has a (direct or indirect) connection with some atom containing a free variable; and, if so, line 6 appends $\mu(\beta)$ to the output.
In particular, $\mathsf{NearCon}$ is  implemented by Algorithm~\ref{alg:connectedness}, which mimics nondeterministic graph reachability: given an atom $\beta$, it checks whether
$\mu(\beta)$ contains a free variable; if so,  then it accepts, else it guesses an ``adjacent'' atom $\alpha$ and, recursively, repeat the former check.

\begin{theorem}\label{thm:CAN-RES}
Problem $\textsc{can}$ belongs to $F\mathsf{PSPACE} \setminus F\mathsf{P}^{\mathsf{PH}}$ in general and to
$F\mathsf{P}$ both in the medial and agile case.
\end{theorem}

\vspace{-3mm}

\begin{proof}[Proof sketch]
(UBs) In the general case, Algorithm~\ref{alg:connectedCan} runs in $F\mathsf{PSPACE}$ using $\mathsf{NearCon}$ as an oracle in $\mathsf{NPSPACE}$. Since $F\mathsf{PSPACE}^ \mathsf{NPSPACE} = F\mathsf{PSPACE}$, the result follows.
In the remaining two cases, Algorithm~\ref{alg:connectedCan} runs in $F\mathsf{P}$ and $\mathsf{NearCon}$ is in  $\mathsf{NL}$. The result follows since $F\mathsf{P}^ \mathsf{NL} = F\mathsf{P}$. 
(LB) By Theorem~\ref{thm:canBound}, we know  that the exponentiality of $\mathit{can}(\unit,\kb)$ is unavoidable in the general case. Hence, such an object cannot be constructed in $F\mathsf{P}^{\mathsf{PH}}$.
\end{proof}

We continue our analysis by presenting   Algorithm~\ref{alg:buildCore} for computing $\mathit{core}(\unit,\kb)$.
This time: line 1 constructs \mbox{$\varphi = \mathit{can}(\unit,\kb)$ and collects its atoms in $A$;} line 2 enumerates each $\alpha \in A$; line 3 builds $\varphi'$ from $\varphi$ by removing $\alpha$; line 4 checks whether $\varphi \longrightarrow \varphi'$; if so, line 5 copies $\varphi'$ in $\varphi$; finally, after removing all redundant atoms, line 6 prints $\varphi$.

\begin{algorithm} [t!]
	\DontPrintSemicolon
	\KwInput{$\kb  = (\db, \varsigma)$ and $\unit = \{\tau_1,...,\tau_m\}$.}
	$d_{{\bf s}_1},...,d_{{\bf s}_n} := \, \tau_1 \otimes ... \otimes \tau_m$
	
	{\bf print} $x_{{\bf s}_1},...,x_{{\bf s}_n} \leftarrow  \top(x_{{\bf s}_1}) \wedge ... \wedge \top(x_{{\bf s}_n})$\;
	\For{$\alpha \in \varsigma(\tau_1) \otimes ... \otimes \varsigma(\tau_m)$}
	{
		\For{$\beta \in \{\alpha\} \cup \mathit{clones}(\alpha)$}
		{
			\If{$\mathsf{NearCon}(\beta,\unit,\kb) == \mathsf{Yes}$}
			{
				{\bf print} $\wedge$  $\mu(\beta)$\;
			}
		}		
	}
	\caption{$\mathsf{BuildCan(\unit,\kb)}$ }\label{alg:connectedCan}
\end{algorithm}
\begin{algorithm}[t!]
	\DontPrintSemicolon
	\KwInput{$\kb  = (\db, \varsigma)$, $\unit = \{\tau_1,...,\tau_m\}$ and an atom $\beta$.}
	{\bf guess}  $d_{\bf s} \in \tau_1 \otimes ... \otimes \tau_m$ \;
	\If{$x_{\bf s} \in \Dom_{\{\mu(\beta)\}}$ }
	{ $\mathsf{accept~this~branch}$ }
	
	{\bf guess} $\alpha \in \varsigma(\tau_1) \otimes ... \otimes \varsigma(\tau_m)$\;
	\If{$\Dom_{\{\mu(\alpha)\}} \cap \Dom_{\{\mu(\beta)\}} \neq \emptyset$}
	{
	    \If{$\mathsf{NearCon}(\alpha,\kb,\unit) == \mathsf{Yes}$}
	    {	 
	        $\mathsf{accept~this~branch}$
	    }
	    
	}
	$\mathsf{reject~this~branch}$
	\caption{$\mathsf{NearCon}(\beta,\unit,\kb)$}\label{alg:connectedness}
\end{algorithm}
\begin{algorithm} [t!]
	\DontPrintSemicolon
	\KwInput{$\kb  = (\db, \varsigma)$ and $\unit = \{\tau_1,...,\tau_m\}$.}
		$\varphi := \mathsf{BuildCan}(\unit,\kb)$ \ \ \ \mbox{and} \ \ \ $A := \mathit{atm}(\varphi)$\;
		
	    \For{$\alpha \in A$}
	    {
	        $\varphi' := \mathit{remove}(\alpha,\varphi)$\;
	        \If{$\varphi \longrightarrow \varphi'$}
	        {$\varphi := \varphi'$}
	    }
	    {\bf print} $\varphi$

	\caption{$\mathsf{BuildCore(\unit,\kb)}$ }\label{alg:buildCore}
\end{algorithm}

\begin{theorem}
$\textsc{core}$ is in $F\mathsf{EXP}^{\mathsf{NP}}\setminus F\mathsf{P}^{\mathsf{PH}}$ in the general case, in $F\mathsf{P}^{\mathsf{NP}}$ in the medial case and in $F\mathsf{P}$ in the agile case. 
Unless $\mathsf{NP}=\mathsf{coNP}$, $\textsc{core} \not\in F\mathsf{P}$ in the medial case.
\end{theorem}

\vspace{-3mm}

\begin{proof}[Proof sketch]
(UBs)  In the medial case, Algorithm~\ref{alg:buildCore} runs in $F\mathsf{P}$
using an oracle in $\mathsf{NP}$ for checking whether
\mbox{$\varphi \longrightarrow \varphi'$} holds. Analogously, in the general case, it runs in $F\mathsf{EXP}$
with an oracle in $\mathsf{NP}$. In the agile case, everything becomes polynomial as both $|\varphi|$ and $\Dom_{\mathit{atm}(\varphi)} \cap \V$ are bounded.
(LBs) In the general case it comes by Theorem~\ref{thm:CAN-RES}.
In the medial case, let $\textsc{cff-core}$ be the $\mathsf{DP}$-complete problem~\cite{DBLP:journals/tods/FaginKP05}:
{\em given a pair $(\varphi',\varphi)$ of constants-free formulas  s.t. $\mathit{atm}(\varphi') \subseteq \mathit{atm}(\varphi)$, is \mbox{$\varphi' \simeq \mathit{core}(\varphi)$}?}
%
%
Let $\mathbb{NT}$ be all constant-free formulas in $\WCQ$ closed under $\top$, and $\textsc{nt-core}$ be $\textsc{cff-core}$ restricted to $\mathbb{NT}$ formulas.
Indeed, $\textsc{nt-core}$ remains $\mathsf{DP}$-hard.
For each $\varphi \in \mathbb{NT}$, it is possible to build in $F\mathsf{P}$ a pair $(\unit,\kb)$  s.t. $\varphi$ $\simeq$ $\mathit{can}(\unit,\kb)$ and $\mathit{core}(\varphi)$ $\simeq$ $\mathit{core}(\unit,\kb)$.
We can now Turing reduce $\textsc{nt-core}$ to $\textsc{core}$. Indeed, from a pair $(\varphi',\varphi)$ of formulas in $\mathbb{NT}$ do: $(i)$ construct in $F\mathsf{P}$ a pair $(\unit,\kb)$  s.t. $\varphi \simeq \mathit{can}(\unit,\kb)$; $(ii)$ construct $\mathit{core}(\unit,\kb)$; and $(iii)$ check in $\mathsf{NP}$ whether $\varphi' \simeq \mathit{core}(\unit,\kb)$. If $\textsc{core}$ was in $F\mathsf{P}$, then step $(ii)$ would also be in $F\mathsf{P}$ and, hence, $\textsc{nt-core}$ would be in $\mathsf{NP}$, which is impossible unless \mbox{$\mathsf{NP} = \mathsf{coNP}$.}
\end{proof}

We conclude by observing that $\mathit{can}(\unit,\kb)$ may be a practical alternative to  $\mathit{core}(\unit,\kb)$ since their size asymptotically coincide but the former is computationally simpler.

\subsubsection{Decision Problems.}

Let $\{\C_f,\C_\ell\}$ be a partition of $\C$ with $\C_\mathit{f}$ being {\em flat} (e.g., $\mathsf{a}$, $\mathsf{a}_1$, $\mathsf{a}_2$) and $\C_{\ell}$ $=$ $\{c^s\!:c\!\in\!\C_f \wedge s>1\}$ being {\em lifted}.
Let $\C_\mathit{re}$ $=$ $\{\mathsf{alias}\}$ $\cup$ $\{\mathsf{a}_i, \mathsf{b}_i\!:\!i\!>\!0\}$
and \mbox{$\C_\mathit{in}$ $=$ $\C_f\!\setminus\!\C_\mathit{re}$} 
denote {\em reserved} and {\em input} flat constants.
Analogously, let $\PS_\mathit{re}$ $=$ $\{\mathsf{focus},\mathsf{arc}, \top\}$ $\cup$ $\{\mathsf{twin}_s\!:\!s\!>\!1\}$ and \mbox{$\PS_\mathit{in}$ $=$ $(\PS \setminus \PS_\mathit{re}) \cup \{\top\}$.}
Inputs of our reductions use symbols only from $\C_\mathit{in}$ $\cup$ $\PS_\mathit{in}$.
Consider any \mbox{$k\!>\!0$} and \mbox{$a\!\in\!\C_\mathit{f}$.}
From \mbox{$C\!\subset\!\C_\mathit{f}$,}
let $\mathit{fc}(C)$ $=$ \mbox{$\{\mathsf{focus}(c)\!:\!c\!\in\!C\}$} 
and
$\mathit{tw}(C,k)$ be the set \mbox{$\{\top(c^s),\mathsf{twin}_s(c,c^s)\!:\!c\!\in\!C$ $\wedge$ $s\in[2..k]\}$}.
Let \mbox{$\langle a \rangle^k$ $=$ $\langle a,a^2,..,a^k\rangle$.}
From a unary \mbox{$\unit$ $=$ $\{\tau_1,..,\tau_m\}$,}
let $\unit^k$ be  $\{\tau_1^k,..,\tau_m^k\}$.
From a dataset $D$, let
$\mathit{double}(D,a)$ be the  \mbox{dataset $\{p(c_1,..,c_{|p|})\!: p(b_1,..,b_{|p|})\!\in\!\db \wedge \mbox{each}~c_i\in g(b_i)\}$,} where \mbox{$g(a)\!=\!\{a,\mathsf{alias}\}$,} and $g(b_i)\!=\!\{b_i\}$ whenever $b_i\!\neq\!a$.
From $\tau\!\in\!\C^k$,
let $\mathit{off}(\tau,a)\!\in\!\C_{f}^k$ be the tuple obtained from $\tau$ by replacing each $c^s$ with $c$ and $\mathsf{alias}$ with $a$.
From two tuples $\langle { \bf t}_1 \rangle$ and $\langle { \bf t}_2 \rangle$ of any arities, let $\langle { \bf t}_1 \rangle \circ \langle { \bf t}_2 \rangle$ be $\langle { \bf t}_1,{ \bf t}_2 \rangle$. 
From two units $\unit_1$ and $\unit_2$ with $\Dom_{\unit_1} \cap \Dom_{\unit_2} = \emptyset$, let $\unit_1 \times \unit_2$ be the unit $\{\tau_1 \circ \tau_2: \tau_1\in \unit_1 \wedge \tau_2 \in \unit_2\}$.

\begin{theorem}\label{thm:DEF}
\textsc{def} is  $\mathsf{coNEXP}$-complete, $\mathsf{coNP}$-complete and in $\mathsf{P}$ in the general, medial and agile case, respectively. In particular, for each $k>0$, LBs hold already for \textsc{def}$_k$.
\end{theorem}

\vspace{-3mm}

\begin{proof}[Proof sketch]

(UBs) From Algorithm~\ref{alg:NotDef},  $\neg\textsc{def}$ is in $\mathsf{NEXP}$ [$\mathsf{NP}$] in the general [medial] case:
 lines 1,4-5 run in exponential [polynomial] time, while
lines 2-3 run in
nondeterministic exponential [polynomial] time. 
%
For the agile case, after replacing  {\bf guess} with  {\bf for} in lines 2-3, the algorithm runs in polynomial time.
(LBs) For each $k>0$, \mbox{\textsc{ess}$_1$ $\leq$ $\neg$\textsc{def}$_k$} via the mapping $(\kb, \unit, \tau) \mapsto (\kb', \unit^k)$,
where 
$\kb$ $=$ $(\db,\varsigma)$, 
$|\unit| > 1$, 
\mbox{$\kb'$ $=$ $(\db',\varsigma')$,}
$\db'$ $=$ $D$ $\cup$ $\mathit{tw}(\Dom_D,k)$ $\cup$ $\mathit{fc}(\Dom_{\unit\cup \{\tau\}})$,
$\varsigma'$ is s.t. $\varsigma'(D',\tau')$ $=$ $\varsigma(D,\tau_0)$ $\cup$ 
$\mathit{tw}(\Dom_{\{\tau_0\}},k)$ $\cup$ $\mathit{fc}(\Dom_{\unit\cup \{\tau\}})$, $\tau' \in \Dom_{D'}^k$ and $\tau_0$ $=$ $\mathit{off}(\tau',\mathsf{dummy})$.
\end{proof}

\begin{theorem}\label{thm:essfinal}
\textsc{ess} is $\mathsf{NEXP}$-complete, $\mathsf{NP}$-complete and in $\mathsf{P}$ in the general, medial and agile case, respectively. 
In particular, LBs  hold for \textsc{ess}$_k$ with $k>0$ even if $|\unit|>1$.
\end{theorem}

\vspace{-3mm}

\begin{proof}[Proof sketch]
(UBs)
Exploit Algorithm~\ref{alg:NotDef} 
after removing line $2$.
(LBs) Let \textsc{php-sb} be the $\mathsf{NEXP}$-complete problem: {\em given a sequence \mbox{${\bf s}$ $=$ $I_1,..,I_{m+1}$} of
sets of atoms of the form $\mathsf{br}(a,b)$ s.t. $a,b \in \C$ 
and \mbox{$\Dom_{I_i} \cap \Dom_{I_{j\neq i}}\!=\! \emptyset$,} is there any homomorphism from \mbox{$I_1\! \otimes\!...\!\otimes\!I_{m}$} to $I_{m+1}$?} 
Let \textsc{{\footnotesize3}-col} be the $\mathsf{NP}$-complete problem: {\em given a graph \mbox{$G$ $=$ $(V,E)$,} is there a map \mbox{$\lambda: V \rightarrow \{\mathtt{r},\mathtt{g},\mathtt{b}\}$}  s.t. \mbox{$\{u,v\}\!\in\!V$} implies \mbox{$\lambda(u)\!\neq\!\lambda(v)$?}}
In the general case, for each \mbox{$k\!>\!0$,}
by adapting a technique of \citeauthor{DBLP:conf/icdt/CateD15}~(\citeyear{DBLP:conf/icdt/CateD15}), \mbox{\textsc{php-sb} $\leq$ \textsc{ess}$_k$} via the map
{\bf s} $\mapsto$ $(\kb, \unit^k, \langle \mathsf{a}_{m+1}\rangle^k)$, where
\mbox{$\kb$ $=$ $(\db,\varsigma)$,} 
\mbox{$\db$ $=$ $\db_1 \cup...\cup \db_4$,}
$\db_1$ $=$ $I_1 \cup ... \cup I_{m+1}$,
$\db_2$ $=$ $\{\top(c):c \in \Dom_{\db_1}\}$,
\mbox{$\db_3$ $=$ $\{\top(\mathsf{a}_i),\mathsf{br}(\mathsf{a}_{i},c):c \in \Dom_{I_i}\}_{i \in[m+1]}$,}
\mbox{$\db_4$ $=$ $\mathit{tw}(\{\mathsf{a}_1,...,\mathsf{a}_{m+1}\},k)$,}
$\varsigma$ always selects the entire datase $D$, and
$\unit$ is the unary unit $\{\langle \mathsf{a}_1 \rangle,...,\langle \mathsf{a}_m \rangle\}$. 
In the medial case, for each \mbox{$k\!>\!0$,} \mbox{\textsc{{\footnotesize3}-col} $\leq$ \textsc{ess}$_k$} via the map
$G$ $\mapsto$ 
$(\kb, \unit^k, \langle \mathsf{b}_1\rangle^k)$, where  \mbox{$\kb$ $=$ $(\db,\varsigma)$,} 
$\db$ $=$ $\db_1$ $\cup\,...\,\cup$ $\db_5$,
$\db_1$ $=$ $\{\mathsf{arc}(u_i,v_i),$ $\mathsf{arc}(v_i,u_i)$ $:\{u,v\}\in E \wedge i\in[2]\}$,
$\db_2$ $=$ $\{\top(c):c \in \Dom_{\db_1}\}$,
$\db_3$ $=$ $\{\top(\mathsf{b}_i),\mathsf{arc}(\mathsf{b}_i,\mathsf{b}_z):i,z\in[4] \wedge i\neq z\}$,
$\db_4$ $=$\! $\{\top(\mathsf{a}_i),\mathsf{arc}(\mathsf{a}_i,v_i),\mathsf{arc}(v_i,\mathsf{a}_i)\!:\!v\in\!V \wedge i\!\in\![2]\}$, $\db_5$ $=$ $\mathit{tw}(\Dom_{\db_3}\cup\Dom_{\unit},k)$, 
$\varsigma$ always selects the entire dataset $D$, and
$\unit$ is the unary unit $\{\langle \mathsf{a}_1 \rangle,\langle \mathsf{a}_2 \rangle\}$.
\end{proof}

Before analyzing \textsc{sim}, \textsc{inc} and \textsc{prec}, we provide upper bounds for our  gadgets. Both problems reduce to $\textsc{ess}$. Given an input $\iota = (\kb,\unit,\tau,\tau')$, to show \mbox{\textsc{gad{\scriptsize1}} {\scriptsize$ \leq$} $\textsc{ess}$} we can use the reduction $\iota \mapsto (\kb, \unit\cup\{\tau'\} ,\tau)$ and for \mbox{\textsc{gad{\scriptsize2}} {\scriptsize$ \leq$} $\textsc{ess}$} we can use the reduction  $\iota \mapsto (\kb, \unit\cup\{\tau\} ,\tau')$.

\begin{algorithm} [t!]
	\DontPrintSemicolon
	\KwInput{$\kb  = (\db, \varsigma)$ and $\unit = \{\tau_1,...,\tau_m\}$.}
	$\varphi := \mathsf{BuildCan}(\unit,\kb)$\;
	
	{\bf guess}  a tuple $\tau \in (\Dom_\db)^n \setminus \unit$ \;
	{\bf guess} a mapping $h$ from $\Dom_{\mathit{atm}(\varphi)}$ to $\Dom_{\varsigma(\tau)}$\;
	
	\If{$h$ \mbox{{\rm is a }}$\C$\mbox{{\rm -homomorphism from}} $\mathit{atm}(\varphi)$ \mbox{{\rm to}} $\varsigma(\tau)$}
	{
	    \If{$\langle h(x_{{\bf s}_1}),...,h(x_{{\bf s}_n}) \rangle = \tau$}
	    {$\mathsf{accept~this~branch}$ }
	 }
	
	$\mathsf{reject~this~branch}$

	\caption{$\mathsf{NotDef}(\unit,\kb)$}\label{alg:NotDef}
\end{algorithm}

\begin{proposition}\label{prop:gad}
Both \textsc{gad{\scriptsize1}} and \textsc{gad{\scriptsize2}} belong to
$\mathsf{NEXP}$, $\mathsf{NP}$ and $\mathsf{P}$ in the general, medial and agile case, respectively.
\end{proposition}

We can now complete our computational trip.

\begin{theorem}\label{thm:sim}
\textsc{sim} is $\mathsf{NEXP}$-complete, $\mathsf{NP}$-complete and in $\mathsf{P}$ in the general, medial and agile case, respectively. In particular, for each $k>0$, LBs hold already for \textsc{sim}$_k$.
\end{theorem}

\vspace{-3mm}

\begin{proof}[Proof sketch]
(UBs) An input for ${\textsc{sim}}$ is $\mathsf{accepted}$ if, and only if, it is $\mathsf{accepted}$ both for \textsc{gad{\scriptsize1}} and  \textsc{gad{\scriptsize2}}.
(LBs) It holds that \textsc{ess}$_1$ $\leq$ \textsc{sim}$_k$ via the map $(\kb, \unit, \tau) \mapsto (\kb', \unit^k,\tau^k,\langle \mathsf{alias}\rangle^k)$, where $\kb$ $=$ $(D,\varsigma)$,
$\unit$ $=$ $\{\tau_1,..,\tau_m\}$,
$m\!>\!1$, 
$\tau_1$ is of the form $\langle c \rangle$,
$\kb'$ $=$ $(D',\varsigma')$,
$D'$ $=$ $\bar{D}$ $\cup$ $\mathit{tw}(\Dom_{\bar{D}},k)$,
$\bar{D}$ $=$ $\mathit{double}(D,c)$,
$\varsigma'$ is  s.t. $\varsigma'(D',\tau')$ $=$ $\varsigma(D,\tau_0)$ $\cup$
$\mathit{tw}(\Dom_{\{\tau_0\}},k)$, $\tau' \in \Dom_{D'}^k$, and $\tau_0$ $=$ $\mathit{off}(\tau',c)$.
\end{proof}

\begin{theorem}\label{thm:inc}
\textsc{inc} is $\mathsf{coNEXP}$-complete, $\mathsf{coNP}$-complete and in $\mathsf{P}$ in the general, medial and agile case, respectively. In particular, for each $k>0$, LBs hold already for \textsc{inc}$_k$.
\end{theorem}

\vspace{-3mm}

\begin{proof}[Proof sketch]
(UBs)
An input for ${\textsc{inc}}$ is  {\em $\mathsf{accepted}$} if, and only if, it is  {\em $\mathsf{accepted}$} both for $\neg$\textsc{gad{\scriptsize1}} and  $\neg$\textsc{gad{\scriptsize2}}.
(LBs) It holds that \textsc{ess}$_1$ $\leq$ $\neg$\textsc{inc}$_k$ via the following map: $(\kb, \unit, \tau) \mapsto (\kb', \unit^k,\tau^k,\langle \mathsf{alias}\rangle^k)$, where \mbox{$\kb$ $=$ $(D,\varsigma)$,}
\mbox{$\unit$ $=$ $\{\tau_1,..,\tau_m\}$,}
$m\!>\!1$,
$\tau_1$ is of the form $\langle c \rangle$,
\mbox{$\kb'$ $=$ $(\db',\varsigma')$,} 
\mbox{$\db'$ $=$ $D$ $\cup$
$\mathit{tw}(\Dom_{\bar{D}},k)$ $\cup$ $\mathit{fc}(\Dom_\db)$,}
$\bar{D} = \mathit{double}(D,c)$,
$\varsigma'$ is  s.t. $\varsigma'(D',\tau')$ $=$ $\varsigma(D,\tau_0)$ 
$\cup$ $\mathit{tw}(\Dom_{\{\tau_0\}},c)$
$\cup$
$\mathit{fc}(\Dom_\db)$, $\tau' \in \Dom_{D'}^k$, and $\tau_0$ $=$ $\mathit{off}(\tau',c)$.
\end{proof}

\begin{theorem}\label{thm:prec}
\textsc{prec} is $\mathsf{DEXP}$-complete, $\mathsf{DP}$-complete and in $\mathsf{P}$ in the general, medial and agile case, respectively. In particular, for each $k>1$, LBs hold already for \textsc{prec}$_k$.
\end{theorem}

\vspace{-3mm}

\begin{proof}[Proof sketch]
(UBs)
An input for ${\textsc{prec}}$ is  {\em $\mathsf{accepted}$} if, and only if, it is  {\em $\mathsf{accepted}$}  both for \textsc{gad{\scriptsize1}} and  $\neg$\textsc{gad{\scriptsize2}}.
(LBs) for each $L\in \mathsf{DEXP}$ (resp., $\mathsf{DP}$), \mbox{$L$ $\leq$ \textsc{prec}$_k$} in the general (resp., medial) case, via the mapping \mbox{$w \mapsto (\kb, \unit, \tau', \tau'')$,} 
where
$w$ is any string over the alphabet of $L$,
$L$ $=$ $L_1$ $\cap$ $L_2$,
$L_1$ $\in$ $\mathsf{NEXP}$ (resp., $\mathsf{NP}$),
$L_2$ $\in$ $\mathsf{coNEXP}$ (resp., $\mathsf{coNP}$),
$\rho_1$ is a reduction from $L_1$ to ${\textsc{ess}}_{1}$ in the general (resp., medial) case,
$\rho_2$ is a reduction from $L_2$ to $\neg{\textsc{ess}}_{k-1}$ in the general  (resp., medial) case,
each \mbox{$\rho_i(w)$ $=$ $(\kb_i, \unit_i, \tau_i)$,}
each \mbox{$\kb_i$ $=$ $(\db_i,\varsigma_i)$,}
w.l.o.g. each constant and predicate different from $\top$ in $\db_1$ (resp., $\db_2$) has $\mathsf{a}$- (resp., $\mathsf{b}$-) as a prefix;
each \mbox{$\unit_i$ $=$ $\{\tau_{i,1},..,\tau_{i,m_i}\}$,}
\mbox{$\kb$ $=$ $(\db,\varsigma)$,}
\mbox{$\db$ $=$ $\db_1\cup\db_2$,}
$\varsigma$ is s.t.
$\varsigma(\db,\tau)$ $=$ $\varsigma_1(\db_1,\tau|_{\Dom(\db_1)})$ $\cup$ $\varsigma_2(\db_2,\tau|_{\Dom(\db_2)})$,
\mbox{$\unit$ $=$ $\unit_1\times\unit_2$,}
\mbox{$\tau'$ $=$ $\tau_1\circ\tau_{2,1}$,}
$\tau''$ $=$ $\tau_{1,1}\circ\tau_2$,
each $\tau|_{\Dom(\db_i)}$ is obtained from $\tau$ by removing the terms not in $\Dom(\db_i)$.
\end{proof}

\section{Related Work}\label{sec:related}

\vspace{-0.5mm}
\subsection{Semantic Similarity and Entity Set Expansion} 
\vspace{-1mm}

In {\em semantic similarity}, the aim is to  measure some distance between entities based on their shared semantic connections; note, however, that while entities like $\mathtt{hammer}$ and $\mathtt{nail}$ are highly related, they are not very similar~\cite{DBLP:journals/csur/ChandrasekaranM21}.
Approaches using co-occurrences~\cite{DBLP:conf/ecml/Turney01,DBLP:conf/ijcai/GabrilovichM07,DBLP:conf/nodalida/Kolb09,DBLP:conf/emnlp/MamouPWEGGIK18} or word embeddings~\cite{DBLP:journals/corr/abs-1301-3781,DBLP:conf/ijcai/XunLZGZ17,DBLP:conf/naacl/AsrZJ18,DBLP:conf/naacl/DevlinCLT19} from corpora or tabular data mainly consider words that occur in similar contexts to be similar.
Other approaches take into account the path-lengths between concepts in a KB~\cite{DBLP:conf/aaai/PedersenPM04,DBLP:conf/aaai/MihalceaCS06,DBLP:journals/pvldb/SunHYYW11,DBLP:conf/aaai/Pirro12} or the distance between entity embeddings in a KG~\cite{DBLP:conf/aaai/LinLSLZ15,DBLP:conf/aaai/LiuZ0WJD020}.

In {\em entity set expansion}, the goal is to expand a given set of entities based on their similarity or relatedness \cite{GoogleSets08,DBLP:conf/www/HeX11}.
This topic is also referred to as {\em entity recommendation}~\cite{DBLP:conf/semweb/BlancoCMT13} or {\em entity suggestion}~\cite{DBLP:conf/ijcai/ZhangXHWWW17}.
Traditional approaches use machine learning or probabilistic models that infer similarities based on the co-occurrence of entities within corpora~\cite{DBLP:journals/tois/BalogBR11,DBLP:conf/nips/GhahramaniH05,DBLP:conf/ijcai/HuangZSWL18,DBLP:conf/emnlp/MamouPWEGGIK18,DBLP:journals/fcsc/ShiDCHWL21,DBLP:conf/icdm/WangC08}. 
Recent approaches based on structured resources primarily use probabilistic models~\cite{DBLP:journals/ws/ChenCZDWW18,DBLP:journals/tkde/JayaramKLYE15,DBLP:conf/www/LissandriniMPV20,DBLP:conf/semweb/Passant10,DBLP:journals/pvldb/SunHYYW11,DBLP:conf/sigir/ZhangCCDWW17} or graph embeddings~\cite{DBLP:conf/aaai/LinLSLZ15} to infer the presence of common properties.

\vspace{-1.5mm}
\subsection{Additional Related Topics and Tasks}
\vspace{-1mm}

We start with {\em formal concept analysis} in applied mathematics~\cite{DBLP:conf/icfca/Wille09}.
Here, the objective is to study properties of 
a (complete) lattice of formal concepts
which essentially coincides to our $\mathit{ex}(\unit,\kb)$ ---together with the characterizations of their elements--- when $\unit$ is unary and $\kb$ consists of both a dataset having only unary predicates and the ``identity'' summary selector returning always the entire dataset.
In constraint programming, the general problem of 
{\em expressibility} ~\cite{DBLP:conf/cp/Willard10} 
asks whether a set $J$ of tuples (analogous to our units but not necessarily proper) can be expressed by a formula $\varphi$ in a certain language (primitive positive formulas or beyond) w.r.t. some relational structure (the analogous of our selective KBs with the identity selector); hence, the question is roughly whether $J$ is $\kb$-definable.
Analogous topics are {\em definability} in symbolic logic~\cite{DBLP:journals/jsyml/BodirskyPT13} and in graph databases~\cite{DBLP:conf/icdt/AntonopoulosNS13}, as well as
{\em reverse engineering} in relational databases~\cite{DBLP:journals/vldb/TranCP14}.
Moreover, {\em query by examples} in relational databases~\cite{DBLP:conf/icdt/Barcelo017}
considers the following task under conjunctive formulas: given a dataset $D$, a set $J$ of tuples (analogous to our units but not necessarily proper), and another set $J^-$, is there a \mbox{$D$-definable} set $J' \supseteq J$  s.t. $J'\cap J^- = \emptyset$?
Likewise, {\em reverse engineering}~\cite{DBLP:conf/ijcai/Gutierrez-Basulto18},
{\em concept induction}~\cite{DBLP:series/ssw/LehmannV14} and
{\em concept learning}~\cite{DBLP:conf/ijcai/FunkJLPW19} in description logics use DL-concepts as explanation language and consider unary tuples.
Finally, the topic of {\em least general generalizations} in 
description logics~\cite{DBLP:conf/aaai/JungLW20} consider tasks (such as {\em least common subsumer} and {\em most
specific concept}) that could be combined to get something analogous to our characterizations. However, in that case, 
tuples are only unary, summaries are not taken into account, and characterizations may not exist (i.e., they may be of infinite size).

\section{Discussion and Conclusion}\label{sec:conclusion}

The proposed framework tailors and combines existing tools to provide a formal semantics to nexus of similarity.
Definitions 1--4 fix the key principles of the framework, while SKBs accommodate  specific needs in different scenarios.
Regarding nearly connected formulas, there are at least three good reasons for this choice.
%
First, they specialize primitive positive first-order formulas, which underlie SQL/SPARQL, while generalizing $\mathcal{ELI}$ concepts that underlie OWL 2 EL, both of which are well-established formalisms.
%
Second, they avoid: disconnected components and disjunction
as they tend to go beyond semantic connections; forcing one connected component or acyclicity as characterizations may not exist; negation and universal quantification as they inherently look beyond summaries; and forcing built-in equality since, for example, it makes definable single entities even if other entities share the same properties.
Third, they allow: constants as they are informative; existential quantification (governing non-free variables) to capture connections not expressible via constants; 
conjunction and joins as they natively allow to express connections;  and multiple free variables to go beyond unary concepts.

We close with some future work directions.
On the theoretical side, it would be valuable to investigate how to incorporate controlled forms of disjunction or arithmetic operations in our explanation language. 
From a practical perspective, the plan is to boost our current prototype to build a web service that implements the proposed framework on top of widely-used KGs such as BabelNet.

\bibliographystyle{IEEEtranSN}
\bibliography{kr-sample}


\newpage

\appendix

\appendix
\onecolumn
\section{Introduction} This appendix is divided into three sections, the first concerns propositions already present in the main paper, the second will give some fundamental results useful for the formal proof of the theorems encountered while the third concerns the formalization of the proofs of the theorems that are mentioned in the main paper for which only a general idea was given.

\section{Proof of Propositions}
In this first part as previously mentioned we will focus on the propositions so that the reader can become more and more familiar with the results.
Where necessary, we will add ad hoc definitions from time to time.
We begin the following section with an immediate result.\\ \\
{\bf Proposition 1.} {\em $\WCQ$ is closed under conjunction.}

\begin{proof}
First of all notice that the definition of connectedness for a structure is equivalent to the following one:%
Any structure of cardinality at most one is  connected and given two connected structures $S$ and $S'$, if $\Dom_{S} \cap \Dom_{S'} \neq \emptyset$, then $S \cup S'$ is  connected too.
 Consider now two $n$-ary open formulas $\varphi_1, \varphi_2\in\WCQ$ and let $\varphi_3$ be their conjunction. By definition $\varphi_3$ is the formula $z_1,...,z_n \leftarrow  \mathit{conj}(\varphi_1'), \mathit{conj}(\varphi_2')$, where $\varphi_{i}'$ is the formula obtained from $\varphi_{i}$ by replacing the $j$-th free variable with $z_j$ and each non-free variable $y$ with $y_{\varphi}$. We want to show $\varphi_3\in\WCQ$. We will show that $\mathit{atm}(\varphi_3) \cup \{\mathsf{free}(z_1,...,z_n)\}$ is connected. Let $S_1=\mathit{atm}(\varphi_1') \cup \{\mathsf{free}(z_1,...,z_n)\}$ and $S_2=\mathit{atm}(\varphi_2') \cup \{\mathsf{free}(z_1,...,z_n)\}$. The result now follows since $\mathit{atm}(\varphi_3) \cup \{\mathsf{free}(z_1,...,z_n)\}=S_1\cup S_2$, both $S_1$ and $S_2$ are connected by hypotheses and $\{z_1,...,z_n\}\subseteq\Dom_{S_1}\cap \Dom_{S_2}$.
\end{proof}
The first thing that can jump to mind immediately after the definition of how the evaluation of a formula is carried out is what is the relationship between evaluating this formula with respect to the totality of knowledge at our disposal or focusing only on the summaries, this naturally leads to the following proposition. \\ \\
{\bf Proposition 2.} {\em
Consider some \marco{selective KB} $\kb =(\db, \varsigma)$ and some formula $\varphi$. It holds that $\mathit{inst}(\varphi,\kb) \subseteq \varphi(D)$.}

\begin{proof}
Consider an $n$-ary formula
$\varphi$ as in Equation~\ref{eq:formula}. 
Let \mbox{$\tau=\langle t_1,...,t_n\rangle$}  be
an arbitrary tuple that belongs to $\mathit{inst}(\varphi,\kb)$.
By definition of $\varphi$-instances, there is a \mbox{$\C$-homomorphism} $h$  from $\mathit{atm}(\varphi)$ to $\varsigma(\tau)$ such that $h(x_1),...,h(x_n) = t_1,...,t_n$. 
By definition of summary selector, $\varsigma(\tau)\subseteq\db$.
Thus, $h$  is also a \mbox{$\C$-homomorphism} from $\mathit{atm}(\varphi)$ to $\db$.
Hence,  $\tau$ belongs to $\varphi(\db)$ as well.  
\end{proof}
The following proposition shows how our choice regarding the evaluation of the formulas leads to the natural interpretation given by the fact that the answers to the conjunction of two other formulas are none other than the intersection of the single answers of the two. \\ \\
{\bf Proposition 3.} {\em
Consider a
\marco{selective KB} $\kb =(\db, \varsigma)$ and two formulas $\varphi_1$ and $\varphi_2$ of the same arity.
It holds that $\mathit{inst}(\varphi_1 \wedge \varphi_2,\kb) = \mathit{inst}(\varphi_1,\kb) \cap \mathit{inst}(\varphi_2,\kb)$.}
\begin{proof}
First of all remember that the conjunction of two $n$-ary open formulas $\varphi_1$ and $\varphi_2$, denoted by $\varphi_1\wedge\varphi_2$, is the formula $z_1,...,z_n \leftarrow  \mathit{conj}(\varphi_1'), \mathit{conj}(\varphi_2')$, where $\varphi_{i}'$ is the formula obtained from $\varphi_{i}$ by replacing the $j$-th free variable with $z_j$ and each non-free variable $y$ with $y_{\varphi}$.
    On the one hand, consider an $n$-ary tuple $\tau$ of the form $\langle t_1,...,t_n \rangle$ and say that $\tau\in\mathit{inst}(\varphi_1 \wedge \varphi_2,\kb)$. By definition this means that it exists a $\C$-homomorphism, call it $h$, from $\mathit{atm}(\varphi_3)$ to $\varsigma(\tau)$ with $h(z_i)=t_i$. From $h$ we can construct two distinct $\C$-homomorphism, $h_1$ from $\mathit{atm}(\varphi_1')$ to $\varsigma(\tau)$ and $h_2$ from $\mathit{atm}(\varphi_2')$ to $\varsigma(\tau)$. Moreover $h_1(z_i)=t_i$ and $h_2(z_i)=t_i$. In particular $h_1$ and $h_2$ are nothing but the restrictions of $h$ to $\Dom_{\mathit{atm}(\varphi_1')}$ and $\Dom_{\mathit{atm}(\varphi_2')}$ respectively. We have just shown $\tau\in\mathit{inst}(\varphi_1',\kb) \cap \mathit{inst}(\varphi_2',\kb)$. Since $\varphi_i'\longleftrightarrow\varphi_i$, exploiting Proposition 4 we know that $\tau\in\mathit{inst}(\varphi_1,\kb) \cap \mathit{inst}(\varphi_2,\kb)$.  On the other hand say that $\tau\in\mathit{inst}(\varphi_1,\kb) \cap \mathit{inst}(\varphi_2,\kb)$. Exploiting Proposition 4 we get that $\tau\in\mathit{inst}(\varphi_1',\kb) \cap \mathit{inst}(\varphi_2',\kb)$, this means that two distinct $\C$-homomorphism, $h_1$ from $\mathit{atm}(\varphi_1')$ to $\varsigma(\tau)$ and $h_2$ from $\mathit{atm}(\varphi_2')$ to $\varsigma(\tau)$ with $h_1(z_i)=t_i$ and $h_2(z_i)=t_i$ do exist. From them we can construct a $\C$-homomorphism $h$ from $\mathit{atm}(\varphi_3)$ to $\varsigma(\tau)$ with $h(z_i)=t_i$. To this aim it suffice to consider $h = \{ t \mapsto h_1(t) ~:~t \in \Dom_{\mathit{atm}(\varphi_1')}\}$ $\cup$ $\{ t \mapsto h_2(t) ~:~t \in \Dom_{\mathit{atm}(\varphi_2')}\}$.
\end{proof}
The following proposition dusts off a classic result about the evaluation of conjunctive queries which is also true under the scope of our framework. \\ \\
{\bf Proposition 4.} {\em
\mbox{If $\varphi_1 \longleftrightarrow \varphi_2$, then $\varphi_1(D) = \varphi_2(D)$ for each  $D$, and $\mathit{inst}(\varphi_1,\kb) = \mathit{inst}(\varphi_2,\kb)$
for each \marco{SKB} $\kb$}.}

\begin{proof}
Assume, without loss of generality, that both $\varphi_1$ and $\varphi_2$ 
have $x_1,...,x_n$ as free variables.
First we prove $\varphi_1(D) = \varphi_2(D)$ holds for any possible dataset $D$. 
Since $\varphi_2 \longrightarrow \varphi_1$ holds by hypothesis, we have a $\C$-homomorphism $h$ from $\mathit{atm}(\varphi_2)$ to $\mathit{atm}(\varphi_1)$.
Let \mbox{$\tau = \langle t_1,...,t_n\rangle$} be a tuple in $\varphi_1(D)$.
By definition of output of a formula over a dataset, there is  a $\C$-homomorphism $h'$ from $\mathit{atm}(\varphi_1)$ to $D$ such that $h'(x_1),...,h'(x_n) = t_1,...,t_n$. 
Clearly, \mbox{$\tilde{h}=h'\circ h$}
is a $\C$-homomorphism from $\mathit{atm}(\varphi_2)$ to $D$ such that $\tilde{h}(x_1),...,\tilde{h}(x_n) = t_1,...,t_n$.
Hence, $\tau$ also belongs to $\varphi_2(D)$. 
In a completely symmetrical way, $\tau\in \varphi_2(D)$ implies that $\tau\in \varphi_1(D)$.
Therefore, the first claim holds.
Conversely, to show that $\mathit{inst}(\varphi_1,\kb) = \mathit{inst}(\varphi_2,\kb)$ holds, it suffices to
consider an arbitrary tuple \mbox{$\tau = \langle t_1,...,t_n\rangle$} and 
observe that $\varsigma(\db,\tau)$ is a dataset.
Indeed, by combining this with the first claim, we easily derive that $\varphi_1(\varsigma(\db,\tau))=\varphi_2(\varsigma(\db,\tau))$ holds. 
\end{proof}
The next result shows that the existence of an interpretation is always guaranteed under our hypotheses. \\ \\
{\bf Proposition 5.} {\em
The given $\kb$-unit $\unit$ is always interpreted by some explanation.}

\begin{proof}
Let $\varphi$ be the formula $x_1,...,x_n \leftarrow \top(x_1),...,\top(x_n)$.
We claim that $\varphi$ interprets $\unit$. 
First, to get $\varphi \in \WCQ$ one can observe that:
$(a)$ by definition of unit, $n$ is greater than zero, implying that $\varphi$ is an open formula; and 
$(b)$ by definition of nearly connected open formulas, $\{\top(x_1),...,\top(x_n)\} \cup \{\mathsf{free}(x_1,...,x_n)\}$ is connected.
Second, to get $\mathit{inst}(\varphi,\kb) \supseteq \unit$ one can observe that:
$(a)$ by definition of summary selector, $\Dom_{\{\tau\}}\subseteq\Dom_{\varsigma(\db,\tau)}$ for each $\tau \in \unit$; and
$(b)$ by definition of dataset, $\top(\tau[i])\in\varsigma(\db,\tau)$ for each $i\in[n]$. 
\end{proof}
The following result shows an important property of characterizations, in fact it is impossible to find two characterizations for the same unit that are not equivalent to each other. \\ \\
{\bf Proposition 6.} {\em
If two explanations characterize $\unit$, then they are equivalent.}

\begin{proof}
Consider two explanations $\varphi'$ and $\varphi''$ characterizing $\unit$.
By contradiction, suppose \mbox{$\varphi' \longleftrightarrow \varphi''$} does not hold. 
Consider the formula $\varphi=\varphi'\wedge\varphi''$. 
By proposition~\ref{prop:instConj},
$\varphi$ interprets $\unit$. 
Moreover \mbox{$\varphi'' \longrightarrow \varphi$} and \mbox{$\varphi' \longrightarrow \varphi$} both hold. Since $\varphi'$  characterize $\unit$, 
\mbox{$\varphi \longrightarrow \varphi'$ holds} and, by $\C$-homomorphism composition, \mbox{$\varphi'' \longrightarrow \varphi'$}. On the other hand, since $\varphi''$ also characterizes $\unit$, \mbox{$\varphi' \longrightarrow \varphi''$} holds, then the two are equivalent but this is absurd. 
\end{proof}
A result as immediate as it is useful is the following, in fact it is indirectly used not only in the proposed algorithms but is the basis of the demonstration of various propositions throughout the article. \\ \\
{\bf Proposition 7.} {\em
The direct product of datasets is a dataset.}
\begin{proof}
    Let $D_1, ..., D_k$ be datasets. By definition their direct product, denoted by $D_1 \otimes...\otimes D_k$, is the structure 
$$P=
\left\{p(\langle{\bf c}_1\rangle \otimes ... \otimes \langle{\bf c}_k\rangle)~:~p({\bf c}_1) \in 
D_1 , ... ,  p({\bf c}_k)\in D_k\right\}.$$
By costruction $\Dom_{P}\subseteq \{d_{c^1,..,c^k}~:~c^1\in\Dom_{D_1},...,c^k\in\Dom_{D_k}\}$.
Since $D_1, ..., D_k$ are datasets, we know that $\top(c^i)\in D_i\forall c^i\in\Dom_{D_i}$, hence $\top(d_{c^1,..,c^k})\in P$
  $\forall c^i\in\Dom_{D_i}, i\in[k]$, that is what we wanted to show.
\end{proof}
The following proposition provides us with an incomparable tool for a series of analyses, in fact it guarantees not only that a characterization always exists but indirectly suggests how to construct such a characterization, providing us with important notions on how to be able, for example, to limit the size of such objects, this aspect will be dealt with in more detail below. \\ \\
{\bf Proposition 8.} {\em It holds that $\mathit{can}(\unit,\kb)$ characterizes $\unit$.}
\begin{proof}
All the symbols to which we will refer in the rest of the proof are detailed in Section~\ref{sec:Explanations} of the main paper.
  Let $\varphi$ $=$ $\mathit{can}(\unit,\kb)$.
The first step we need to take is to show that $\varphi$ interprets $\unit$. In order to do so, since \mbox{$\varphi\longrightarrow \Phi$,} it suffices to show that $\mathit{inst}(\Phi,\kb) \supseteq \unit$. By construction, $\mu^{-1}$ is a \mbox{homomorphism} from $\Phi$ to $P \cup C$, in particular the reader may wonder what guarantees this and in the first place if we are really able to reverse this mapping given the definition of $\mu$, locally this is the case since $\mu$ acts as a bijection between $\Phi$ and $P \cup C$, moreover $\mu^{-1}$ does act as a \mbox{homomorphism} since by costruction it preserves every atom. It can be noticed that each $h_i$ $=$ $\{a \mapsto a~:~a \in  \Dom_D \}$ $\cup$ $\{d_{\bf s} \mapsto {\bf s}[i]~:~d_{\bf s} \in \Dom_P\}$ is a homomorphism from $P \cup C$ to $\varsigma(\tau_i)$, each of these homomorphisms in fact is nothing more than the projection on the $i$-th component. Let $A=\Dom_C \setminus \Dom_P$. Since $\mu^{-1}(c)$ is either $c$  when $c\in A$ or of the form $d_{c,...,c}$ otherwise, and since $h_i(c)=c$ and $h_i(d_{c,...,c})=c$, it holds that each $\lambda_i = h_i \circ \mu^{-1}$ is a \mbox{$\C$-homomorphism} from $\mathit{atm}(\Phi)$ to $\varsigma(\tau_i)$. In particular by construction we have $\langle \lambda_i(x_{{\bf s}_1}), ..., \lambda_i(x_{{\bf s}_n})\rangle = \tau_i$.

Second, we need to show that $\varphi' \longrightarrow \varphi$ whenever $\varphi'$ interprets $\unit$. First notice that if $\varphi' \longrightarrow \varphi''$, with $\varphi'\in \WCQ$, we get that $\varphi' \longrightarrow \Tilde{\varphi}$ where $\Tilde{\varphi}$ is the nearly connected part of $\varphi''$, this is true as homomorphisms preserve the connection property of structures in general. Since $\varphi$ is the nearly connected part of $\Phi$ we will just show that $\varphi' \longrightarrow \Phi$ whenever $\varphi'$ interprets $\unit$, this together with the fact that homomorphisms do preserve connectedness will give us the result we want since if $\varphi'$ interprets $\unit$ then in particular $\varphi'\in\WCQ$. So consider now an arbitrary formula $\varphi'$ that interprets $\unit$ and let $z_1,...,z_n$ be its free variables. 
Since, by definition, $\mathit{inst}(\varphi',\kb) \supseteq \unit$ holds, for each $i \in \{1,...,m\}$, we can consider a homomorphism $h_i$ that maps $\mathit{atm}(\varphi')$ to $\varsigma(\tau_i)$ such that $\langle h_i(z_1),...,h_i(z_n) \rangle = \tau_i$.
Let us now also consider the following mapping $\xi = \{ t \mapsto d_{h_1(t),\ldots, h_m(t)}~:~t \in \Dom(\varphi')\setminus A\}$ $\cup$ $\{ t \mapsto t ~:~t \in \Dom(\varphi')\cap A\}$, since it preserves every atom it is a homomorphism from $\mathit{atm}(\varphi')$ to $P \cup C$. Considering now $h= \mu \circ \xi$, we get a $\C$-homomorphism from $\mathit{atm}(\varphi')$ to $\mathit{atm}(\Phi)$ such that each $h(z_i)$ equals to $x_{{\bf s}_i}$, in order to see that it is in fact the case that $h$ is a $\C$-homomorphism it suffices to see how constants do behave, in particular we will first map $c\in\C$ to $d_{c,\ldots,c}$ and then from $d_{c,\ldots,c}$ to $c$ again. Now our proof is completed.

\end{proof}
Next we find another property of characterizations that turns out to be very useful when it comes to definability. \\ \\
{\bf Proposition 9.} {\em
The given unit $\unit$ is $\kb$-definable if, and only if, any explanation that characterizes $\unit$ also defines it.}

\begin{proof}
For the ``only if'' part, having any explanation $\varphi$ that both characterizes and defines $\unit$ immediately implies that $\unit$ is $\kb$-definable due to $\varphi$.
Consider now the ``if'' part.
By Definition~\ref{def:def}, 
since $\unit$ is $\kb$-definable, there is a formula $\varphi'$ (possibly non characterizing $\unit$) such that $\mathit{inst}(\varphi',\kb)=\unit$.
By Theorem~\ref{thm:characterized},
there is some explanation $\varphi$ that characterizes $\unit$. 
By Corollary~\ref{cor:char}, since $\varphi'$ interprets $\unit$, we have that  $\varphi'\longrightarrow \varphi$.
Hence, $\mathit{inst}(\varphi,\kb)\subseteq \mathit{inst}(\varphi',\kb)$ holds. 
Thus, $\unit\subseteq \mathit{inst}(\varphi,\kb)\subseteq \mathit{inst}(\varphi',\kb)=\unit$. 
Since $\mathit{inst}(\varphi,\kb) =\unit$ does not depend on the choice of $\varphi$, it holds for
any explanation that characterizes $\unit$.
\end{proof}
We have already introduced some basic properties on definability but the time has come to ask ourselves if we can say, for example, how many and if there are always definable units, the following result partially answers this question. \\ \\
{\bf Proposition 10.} {\em
$\mathit{ex}(\unit,\kb)$ is finite and never empty.}
\begin{proof}
By following the same argument used in the proof of Proposition~\ref{prop:interpreted}, 
we get that $\mathbf{T}$ is $\kb$-definable: the explanation $\varphi = x_{1},...,x_{n}\leftarrow \top(x_{1})\wedge...\wedge \top(x_{n})$ always defines it.
Hence, $\mathit{ex}(\unit,\kb)$ is never empty.
Also, since the cardinality of the power set of $\mathbf{T}$ is an upper bound on $|\mathit{ex}(\unit,\kb)|$, the latter  is  finite.
\end{proof}
The following proposition provides us with a very useful tool, in practice we know that there is always and only one smallest definable unit which is an expansion of the starting unit. \\ \\
{\bf Proposition 11.} {\em
    It holds that $\mathit{ess}(\unit,\kb) \in \mathit{ex}(\unit,\kb)$.}

\begin{proof}
Consider $\unit'$ and $\unit''$ both in $\mathit{ex}(\unit,\kb)$.
It suffices to show that $\unit'\cap \unit''$ also belongs to $\mathit{ex}(\unit,\kb)$.
Indeed, if this is the case,
by Proposition~\ref{prop:dex}, we can apply the same argument finitely many times and get $\mathit{ess}(\unit,\kb) \in \mathit{ex}(\unit,\kb)$.
To show that $\unit'\cap \unit''$ belongs to $\mathit{ex}(\unit,\kb)$ we need to show that 
there exists a formula $\varphi$ defining it, namely $\mathit{inst}(\varphi,\kb) = \unit'\cap \unit''$.
To this aim, consider two arbitrary explanations 
$\varphi'$ and $\varphi''$  that $\kb$-define $\unit'$ and $\unit''$, respectively.
Let $\varphi = \varphi' \wedge \varphi''$.
The claim holds since, by Proposition~\ref{prop:instConj}, we know that 
\mbox{$\mathit{inst}(\varphi,\kb) = \mathit{inst}(\varphi',\kb) \cap \mathit{inst}(\varphi'',\kb) = \unit'\cap \unit''$.}
\end{proof} 
The following result satisfies a need which is always met whenever an operation on an equivalence class is defined starting from any of its representatives. \\ \\
{\bf Proposition 12.} {\em
Function $\delta$ in $\mathit{eg}(\unit,\kb)$ is well-defined.}

\begin{proof}
    By definition $\delta$ maps every $[\varphi] \in N$ to the set $\mathit{inst}(\varphi,\kb) \setminus \{\tau \in \mathit{inst}(\varphi',\kb):([\varphi'],[\varphi]) \in A\}$, what we need to show is that this result does not depend on the chosen representative. To this aim consider two distinct formulas $\varphi_1$ and $\varphi_2$ both in $[\varphi]$ and two distinct formulas $\varphi_3$ and $\varphi_4$ both in $[\varphi']$. By definition we get $\varphi_1\simeq\varphi_1$ and $\varphi_3\simeq\varphi_4$, hence by Proposition~\ref{prop:equinst} $\mathit{inst}(\varphi_1,\kb)=\mathit{inst}(\varphi_2,\kb)$ and $\mathit{inst}(\varphi_3,\kb)=\mathit{inst}(\varphi_3,\kb)$. This shows that the set $\mathit{inst}(\varphi,\kb) \setminus \{\tau \in \mathit{inst}(\varphi',\kb):([\varphi'],[\varphi]) \in A\}$ is well defined and consequently $\delta$ well-defined.
\end{proof}
We conclude this section with a result that will be extremely useful for showing the upper bounds of the various theorems concerning the exact complexity classes of the decision problems we have defined. \\ \\
{\bf Proposition 13.} {\em Both \textsc{gad{\scriptsize1}} and \textsc{gad{\scriptsize2}} belong to
$\mathsf{NEXP}$, $\mathsf{NP}$ and $\mathsf{P}$ in the general, medial and agile case, respectively.}
\begin{proof}
    Let $\iota = (\kb,\unit,\tau,\tau')$ be an input for \textsc{gad{\scriptsize1}}. To show \textsc{gad{\scriptsize1}} {\scriptsize$ \leq$} $\textsc{ess}$ we can use the reduction $\iota \mapsto (\kb, \unit\cup\{\tau'\} ,\tau)$, please note that this reduction preserves the hypothesys of the problem, in other word the general case reduces to the general, medial and agile case reduce to general, medial and agile case respectively. The if and only if holds by definition. This, together with the upper bounds for $\textsc{ess}$, in the general, medial and agile case respectively, shown in Theorem~\ref{thm:essfinal} ends the first part of what we wanted to show. Let now $\iota = (\kb,\unit,\tau,\tau')$ be an input for \textsc{gad{\scriptsize2}}. To show \textsc{gad{\scriptsize2}} {\scriptsize$ \leq$} $\textsc{ess}$ we can simply use the reduction $\iota \mapsto (\kb, \unit\cup\{\tau\} ,\tau')$ and then use a similar argument as before.
    
\end{proof}
In the next section we will give some preliminary results useful for the proofs of theorems and in addition we will recover some well-known notions.\\ \\
\section{Gadgets and Lemmas}

\noindent The main role of the following section is to provide small preliminary results useful in subsequent demonstrations but also to provide where possible some useful insights for the reader so that some new definitions sound natural. Let's start by recapping some useful basics below.
\begin{definition}
		An (\emph{undirected}) graph $\mathcal{G}$ is a ordered pair $( V,E )$ where $V= \{ v_1,\ldots,v_n \}$, for some fixed $n\in\mathbb{N}$, is the set of its vertices while $E\subseteq V\times V$ is the set of its (undirected) edges, that is $(i,j)\in E \implies (j,i)\in E$.
	\end{definition}
 We recall below when two vertices are said to be adjacent.
	\begin{definition}
		Given a graph $\mathcal{G}=(V,E)$ two distinct vertices, $i,j\in V$, are said to be \emph{adjacents} if there exist $e\in E$ such that $e=(i,j) \vee e=(j,i) $.
	\end{definition}
 We absolutely cannot do without the definition of k-colorability but before providing it it is necessary to remember what a k-coloring is.
	\begin{definition}
		A \emph{$k$-coloring} (or \emph{$k$-vertex coloring}), for some fixed $k\in \mathbb{N}$, of a graph $\mathcal{G}=(V,E)$ is a mapping $h:V\mapsto [k] $ such that no two adjacents vertices are mapped to the same element.
	\end{definition}
 Let us also briefly recall when a graph is said to be complete.
 \begin{definition}
     A graph $\mathcal{G}=(V,E)$ is \emph{complete} if , however taken two distinct vertices $i,j\in V$, they are adjacent.
 \end{definition}
 We are now ready to give a definition of k-colorability.
	\begin{definition}
		Given a graph $\mathcal{G}$ the \emph{k-colorability} problem asks if it is possible to find a k-coloring for $\mathcal{G}$.
	\end{definition}   
 The following definition is not typical, the tool we are presenting will prove to be extremely useful in the following.
	\begin{definition}
		Given a graph $\mathcal{G}=(V,E)$ the graph $\mathcal{G}^{a}=(V^{a},E^{a})$, where $V^{a}=V\cup\{ a\}$ (with "$a$" fresh vertex) and $E^{a}=E\cup\{(a,v)~:~v\in V \}\cup\{(v,a)~:~v\in V \}$, is called \emph{complement to a} of $\mathcal{G}$.
	\end{definition}
	We now recall the following important result from the literature.
	\begin{remark}
		It is well known \cite{Kar72}  that the problem of \emph{k-Col} for $k\geq 3$ is NP-complete for undirected graphs.  
	\end{remark}
	
	It is also well known that a graph $\mathcal{G}=(V, E)$ is n-colourable if and only if it exists a homomorphism from $\mathcal{G}$ to $K_n$, where $K_n$ is the complete graph of order n.
\begin{lemma}\label{lem:color}
		A graph $\mathcal{G}=(V,E)$ is $k$-colourable if and only if the complement to a of  $\mathcal{G}$ is $(k+1)$-colourable.
	\end{lemma}
	\begin{proof}[Proof of Lemma~\ref{lem:color}]
		$\Longrightarrow$ Say that $\mathcal{G}$ is $k$-colourable, this means that it exists $h:V\mapsto [k]$ such that no two adjacents vertices are mapped to the same element. From $h$ we can construct $\bar h:V^{a}\mapsto [k+1]$ in the following fashion:
		$$
		\bar h(v)=
		\begin{cases}
			h(v), \ \ \ \forall v\in V\\
			k+1, \ \ \mbox{otherwise}
		\end{cases}
		$$
		$\Longleftarrow$ Say that $\mathcal{G}^{a}$ is (k+1)-colourable, this means that it exists $\bar h:V^{a}\mapsto [k+1] $ such that no two adjacents vertices are mapped to the same element. Without loss of generality we may say that $\bar h(a)=k+1$. From $\bar h$ we can construct $ h:V\mapsto [k]$ simply by putting
		$$
		h(v)=
		\bar h(v), \forall v\in V
		$$ in fact, since a is adjacent to any other vertex in $V^{a}$ by construction if $\bar h$ is a proper (k+1)-colouring then no other element of $V^{a}$ is mapped to k+1, but this means that every element of $V$ is mapped to $[k]$ by $\bar h$ and thus $h$ is a proper k-colouring of $\mathcal{G}$. 
	\end{proof}
 The following is an almost immediate result but fundamental to the rest of our proofs.
\begin{lemma}\label{lem:homo-composition}
		Consider a $\mathit{T}$-homomorphism $h_1$ from $I_1$ to $I_2$ and 
		a $\mathit{T}$-homomorphism $h_2$ from $I_2$ to $I_3$.
		The composition $h_2\circ h_1$ is a $\mathit{T}$-homomorphism from $I_1$ to $I_3$.    
	\end{lemma}
	
	\begin{proof}
		Lets consider $h=h_2\circ h_1$.
		By definition a $\mathit{T}$-homomorphism from $I_1$ to $I_3$
		is any mapping $h:\Dom(I_1) \rightarrow \Dom(I_3)$ such that:
		\begin{itemize}
			\item for each $t \in T \cap \Dom(I_1)$, $h(t) = t$
			\item for each $p({\bf t}) \in I_1$, $p(h({\bf t})) \in I_3$
		\end{itemize}
		We will prove the two conditions separately.
		For the first one since both $h_1$ and $h_2$ are $\mathit{T}$-homomorphism we have for each $t \in T \cap \Dom(I_1)$, $h(t)=h_2(h_1(t))=h_2(t)=t$, since $h_1$ is a $\mathit{T}$-homomorphism from $I_1$ to $I_2$ we get $t \in T \cap \Dom(I_2)$ (otherwise the map would not be well defined) then in particular since $h_2$ is a $\mathit{T}$-homomorphism we get $h_2(t)=t$ and this end the first part.
		For the second part we need to prove for each $p({\bf t}) \in I_1$, $p(h({\bf t}))=p(h_2(h_1({\bf t}))) \in I_3$, since $h_1$ is a $\mathit{T}$-homomorphism we know for each $p({\bf t'}) \in I_1$, $p(h_1({\bf t'})) \in I_2$ and since $h_2$ is a $\mathit{T}$-homomorphism we know for each $p({\bf t}) \in I_2$, $p(h_2({\bf t})) \in I_3$ but in particular this is true for all $p({\bf t})=p(h_1({\bf t'})) \in I_2$ and this completes the proof.
	\end{proof}
	To better understand the proof of the last theorem present in this appendix, the elements proposed below will be of fundamental importance.
	
	\begin{definition}\label{def:sep}
 Let ${\bf x}_1$ and ${\bf x}_2 $ be the sequences $x_1,\dots,x_{n_1}$ and $x_1',\dots,x_{n_2}'$ respectively.
	    	A formula $\varphi\in\WCQ$ is said \emph{separable} if it is of the form $ {\bf x}_1, {\bf x}_2  \leftarrow  \mathit{conj}(\varphi_1),\mathit{conj}(\varphi_2)$ where $\mathit{dom}(\mathit{atm}(\phi_1))\cap \mathit{dom}(\mathit{atm}(\phi_2))=\emptyset$ and $ {\bf x}_1  \leftarrow  \mathit{conj}(\varphi_1)$ and $ {\bf x}_2  \leftarrow \mathit{conj}(\varphi_2)$ are both in $\WCQ$. 
	\end{definition}
	The next results follow immediately by definition.
	  	 \begin{proposition}\label{prop:bipQuery}
	Consider a separable formula $\varphi$ of the form $ {\bf x}_1, {\bf x}_2  \leftarrow  \mathit{conj}(\varphi_1),\mathit{conj}(\varphi_2)$,
	a $|{\bf x}_1|$-ary tuple $\langle {\bf c}_1 \rangle$,  a $|{\bf x}_2|$-ary tuple $\langle {\bf c}_2 \rangle$, and a knowledge base $\kb$.
	  It holds that $\langle {\bf c}_1, {\bf c}_2 \rangle \in \mathit{inst}(\varphi,\kb)$ if, and only if,  $\langle {\bf c}_1\rangle\in\mathit{inst}(\varphi_1,\kb)$ and $\langle {\bf c}_2\rangle\in\mathit{inst}(\varphi_2,\kb)$.
\end{proposition}
From the previous proposition we can directly derive this corollary.
\begin{cor}\label{cor:bipQuery}
	Consider a separable formula $\varphi$ of the form $ {\bf x}_1, {\bf x}_2  \leftarrow  \mathit{conj}(\varphi_1),\mathit{conj}(\varphi_2)$, a $(|{\bf x}_1|+|{\bf x}_2|)$-ary $\unit$, and a knowledge base $\kb$.
	  Let $\unit_1$ the $|{\bf x}_1|$-ary unit of the form $\{ \tau ~:~ \exists \tau' \ s.t \ \tau\circ\tau' \in \unit \}$ and $\unit_2$ the $|{\bf x}_2|$-ary unit of the form $\{ \tau ~:~ \exists \tau' \ s.t \ \tau'\circ\tau \in \unit \}$. It holds that $\varphi$ is a characterization for $\unit$ if, and only if, $\varphi_1$ is a characterization for $\unit_1$ and $\varphi_2$ is a characterization for $\unit_2$. 
\end{cor}
\begin{definition}
    Given a structure $S$, its \emph{characteristic formula} $\varphi_{S}$ is the formula $\bigwedge_{p(t_1,..,t_l)\in S}p(y_{t_1},..,y_{t_l})$.
\end{definition}
	The next section will see the full proof of the various theorems encountered in the main paper for which only the general ideas of the proofs have been given.

\section{Theorems}
This section has the task of complementing the partial tests present in the main paper, for this purpose the constructions or algorithms will not be repeatedly repeated, rather we will refer directly to elements already present in the paper whenever necessary. Below is a table useful for summarizing and better appreciating the results that we will demonstrate in the rest of the section.

\begin{table}[h]
\centering
\renewcommand{\arraystretch}{1.5}
\begin{tabular}{lccc}
		\hline
		~~~~~{\bf  Problem }~~~~ & {\bf General} & {\bf Medial}  & {Agile} \\
		\hline
  ~~~~~\textsc{core}  & in $F\mathsf{EXP}^{\mathsf{NP}} \setminus F\mathsf{P}^{\mathsf{PH}}$  & in $F\mathsf{P}^{\mathsf{NP}} \setminus F\mathsf{P}^{\dagger}$ & in $F\mathsf{P}$\\
  
		~~~~~\textsc{can}  & ~~~~~in $F\mathsf{PSPACE} \setminus F\mathsf{P}^{\mathsf{PH}}$~~~~~ & in $F\mathsf{P}$ & ~~~~~in $F\mathsf{P}$~~~~~ \\

  \hline
		~~~~~\textsc{def}  &  $\textrm{coNEXP}$-complete  & ~~~~~$\textrm{coNP}$-complete ~~~~~ &  in $\textrm{P}$  \\ 
		~~~~~\textsc{ess} &  $\textrm{NEXP}$-complete   & $\textrm{NP}$-complete  & in $\textrm{P}$  \\
		~~~~~\textsc{prec} &   $\textrm{DEXP}$-complete  &  $\textrm{DP}$-complete   &  in $\textrm{P}$  \\ 
		~~~~~\textsc{sim} &  $\textrm{NEXP}$-complete   & $\textrm{NP}$-complete   &  in $\textrm{P}$  \\ 
		~~~~~\textsc{inc} &  $\textrm{coNEXP}$-complete   & $\textrm{coNP}$-complete   &  in $\textrm{P}$  	 \\ \hline
	\end{tabular}
 \caption{Computational complexity under the three different theoretical assumptions. The lower bound marked by $\dagger$ holds under the assumption that $\mathsf{NP}\neq\mathsf{coNP}$ }
	\label{tab:complexityIntro}\normalsize
\end{table}
\renewcommand{\arraystretch}{1}

Theorem 1 is not present as the Section~\ref{sec:Explanations} of the main paper contains a constructive proof of it, hence the first theorem of which we will deal with the proof in full is the following.

\noindent{\bf Theorem 2.} {\em $\mathit{eg}(\unit,\kb)$ is a directed acyclic graph; 
the sets of direct instances form a partition of $\mathbf{T}$;
$[\mathit{core}(\unit,\kb)]$ is the only source node; and $\delta([\mathit{core}(\unit,\kb)]) = \mathit{ess}(\unit,\kb)$.}
\begin{proof}
   $(1)$ Graph $\mathit{eg}(\unit,\kb)$ is directed by construction; 
%
%
%
%
to see its acyclicity, note that the presence of a cycle would imply that all involved  nodes would be the same equivalence class, which is not possible.
%
To notice this, let us consider a cycle of the form $[\varphi_1],\ldots,[\varphi_n],[\varphi_1]$ with $([\varphi_i],[\varphi_j])\in A$ and $j\cong i+1 \pmod{n}$, where this indicates the usual residue class modulo $n$. By $\C$-homomorphism composition we will get that all element are equivalent, hence, since all of them are minimal with respect to homomorphic equivalence by hypotheses since all of them are cores, they are all isomorphic, hence they are in the same equivalence class.
$(2)$ Let $\tau \in \mathbf{T}$. By construction, there is $[\varphi] \in N$ s.t.  $[\varphi] = [\mathit{core}(\unit \cup \{\tau\},\kb)]$. Hence, $\varphi$ characterizes $\unit \cup \{\tau\}$ and $\tau \in \mathit{inst}(\varphi,\kb)$.
First we claim that $\tau \in \delta([\varphi])$. Suppose by contradiction that there is some  $([\varphi'], [\varphi])\in A$ such that $\varphi \longrightarrow \varphi'$, $\varphi' \longarrownot\longrightarrow \varphi$, and $\tau \in \mathit{inst}(\varphi',\kb)$. 
Since $\tau \in \mathit{inst}(\varphi',\kb)$ and $[\varphi']\in N$ we know that  $\varphi'$ interprets  $\unit \cup \{\tau\}$.
But this is absurd since if $\varphi' \longarrownot\longrightarrow \varphi$ then it violates Corollary~\ref{cor:char}.
Assume now there is $[\varphi'] \in N$ different from $[\varphi]$  such that $\tau \in \delta([\varphi'])$.
Again, since $\tau \in \mathit{inst}(\varphi',\kb)$ and $[\varphi']\in N$ we know that  $\varphi'$ interprets  $\unit \cup \{\tau\}$, hence, by Corollary~\ref{cor:char},  we know that
$\varphi' \longrightarrow \varphi$.
Thus, there is a path from $[\varphi]$ to $[\varphi']$ of the form $[\varphi_1],\ldots,[\varphi_n]$ with $([\varphi_i],[\varphi_{i+1}])\in A$, $\varphi_1\in[\varphi]$, $\varphi_n\in[\varphi']$ and all elements in the path are interpretations for $\unit \cup \{\tau\}$.
This means that  $\tau \not\in \delta([\varphi])$, which is a contradiction.
%
%
$(3)$ To show that $[\mathit{core}(\unit,\kb)]$ is the only source, note that such a node is equal to $[\mathit{core}(\unit \cup \{\tau\},\kb)]$ (being in $N$ by definition), for each $\tau \in \unit$; its uniqueness derives from Corollary~\ref{cor:char} since, for each $\tau \in {\bf T}$, formula  $\mathit{core}(\unit\cup\{\tau\},\kb)$  interprets $\unit$ while $\mathit{core}(\unit,\kb)$ characterizes $\unit$.
$(4)$ Finally, note that $\delta([\mathit{core}(\unit,\kb)]) = \mathit{inst}([\mathit{core}(\unit,\kb)],\kb)$ since $[\mathit{core}(\unit,\kb)]$ is a source; thus,
$\mathit{inst}([\mathit{core}(\unit,\kb)],\kb) = \mathit{ess}(\unit,\kb)$ by Propositions~\ref{prop:def-car} and~\ref{thm:ess}.
%
%

\end{proof}
Below we report another interesting result regarding the possibility of comparing the sizes of the two formulas that we have defined. In fact, the present theorem suggests that, since they are often comparable in size, the calculation of the first, being simpler, is preferable in practice.
\\

\noindent{\bf Theorem 3.} {\em 	Let $c$ be the constant $2^\omega$. It holds that

\[
|\mathit{core}(\unit,\kb)|\leq |\mathit{can}(\unit,\kb)| \leq  c \, \cdot \!\! \prod_{i\in [m]}  |\varsigma(\tau_i)|   \leq  c \cdot |D|^m.
\]

\noindent In particular, there exists a family $\{(\unit_{\bar{m}},\kb_{\bar{m}})\}_{{\bar{m}}>0}$, where $\unit_{\bar{m}} = \{\bar{\tau}_1,...,\bar{\tau}_{\bar{m}}\}$ is a unary unit and  $\kb_{\bar{m}} = (D_{\bar{m}}, \bar{\varsigma})$ 
%
%
is \marco{an SKB} with $\bar{\varsigma}(D_{\bar{m}},\bar{\tau}_m) \subset \bar{\varsigma}(D_{\bar{m}+1},\bar{\tau}_{{m}+1})$, such that

\[
|\mathit{core}(\unit_{\bar{m}},\kb_{\bar{m}})| =  |\mathit{can}(\unit_{\bar{m}},\kb_{\bar{m}})| = 2^{1-\bar{m}} \, \cdot \!\! \prod_{i\in [{\bar{m}}]}  |\bar{\varsigma}(\bar{\tau}_i)|.
\]

\begin{proof}
As partially shown in the main paper the two following parts hold, in particular the upper bound hold by definition of $\mathit{can}(\unit,\kb)$, in fact, from its explicit construction it is evident that the maximum number of atoms in the construction is exactly the one reported. To prove the lower bound, let $\pr_{\bar{m}}$ be the \mbox{$\bar{m}$-th} prime number 
and let $\Gamma_{i}$ $=$ $\{r(c_1^i,c_2^i),$  $...,$ $r(c_{\pr_{i}-1}^i,c_{\pr_{i}}^i),$ $r(c_{\pr_i}^i,c_{1}^i)\}$ $\cup$ $\{\top(c_j^i): j\in [\pr_i]\}$. We define $D_{\bar{m}}$ $=$ $\cup_{i \in [\bar{m}]} \Gamma_{i}$, $\bar{\varsigma}$ s.t.
$\bar{\varsigma}(D_{\bar{m}},\bar{\tau}_i) = \Gamma_i$,
and $\bar{\tau}_i = \langle c_1^i\rangle$. 
Let $\rho$ be a function replacing each occurrence of a $c$ with $x$.
Accordingly, $\mathit{core}(\unit_{\bar{m}},\kb_{\bar{m}})$ $=$  $\mathit{can}(\unit_{\bar{m}},\kb_{\bar{m}})$ is the formula $x_1^\ell \leftarrow \rho(\varsigma_{\ell})$ where $\ell = \pr_1 \cdot...\cdot \pr_{\bar{m}}$, to prove that this assertion is true it is enough to note that we are essentially doing two things:\begin{itemize}
    \item We are constructing a formula formed by a single large cycle whose order is the product of all the primes encountered.
    \item For each variable $y$ in the domain of the formula we are adding $\top(y)$ to the formula itself.
\end{itemize}
from the construction it is therefore evident that the number of atoms is that and to make sure that this formula is a core just note the two following things:\begin{itemize}
    \item The core of a single (finite) cycle always coincides with the cycle itself.
    \item Characterizations must always be closed under top.
\end{itemize} 
with this we have shown both elements in full.
\end{proof} 
The previous theorem suggests an important question in terms of the need to build a structure like the core, is it really always the case to focus on the latter? What are the actual advantages and disadvantages encountered before and after computation? The following two theorems will give us a more precise idea of what we should do.

\medskip

\noindent{\bf Theorem 4.} {\em 	Problem $\textsc{can}$ belongs to $F\mathsf{PSPACE} \setminus F\mathsf{P}^{\mathsf{PH}}$ in general and to
$F\mathsf{P}$ both in the medial and agile case.}
\begin{proof}
Below we will refer to the algorithms in the order in which they were presented in the paper.
First, let's worry about upperbounds and in order to do so lets first consider the general case Algorithm~\ref{alg:connectedCan} runs in $F\mathsf{PSPACE}$ using $\mathsf{NearCon}$ as an oracle in $\mathsf{NPSPACE}$, this is due to the fact that he is not directly manifesting the exponentially large structure linked to the direct product, but is considering one at a time each atom present within it. As we know that $F\mathsf{PSPACE}^ \mathsf{NPSPACE} = F\mathsf{PSPACE}$ the result follows directly.
Concerning the remaining two cases we can see that, since $m$ is bounded, both the following conditions do hold:
\begin{itemize}
    \item Algorithm~\ref{alg:connectedCan} runs in $F\mathsf{P}$.
    \item $\mathsf{NearCon}$ is in  $\mathsf{NL}$.
\end{itemize}Finally the result follows since  we do know that $F\mathsf{P}^ \mathsf{NL} = F\mathsf{P}$ hold. 

As for the lowerbounds instead we do know, by Theorem~\ref{thm:canBound}, that the exponentiality of $\mathit{can}(\unit,\kb)$ is unavoidable in the general case. This obviously leads us to conclude that such an object cannot be constructed in $F\mathsf{P}^{\mathsf{PH}}$, and this concludes the proof.
\end{proof}

The next result ends this first part before entering into the heart of the reductions that ensure completeness results with respect to the computational classes considered.

\medskip

\noindent{\bf Theorem 5.} {\em $\textsc{core}$ is in $F\mathsf{EXP}^{\mathsf{NP}}\setminus F\mathsf{P}^{\mathsf{PH}}$ in the general case, in $F\mathsf{P}^{\mathsf{NP}}$ in the medial case and in $F\mathsf{P}$ in the agile case. 
Unless $\mathsf{NP}=\mathsf{coNP}$, $\textsc{core} \not\in F\mathsf{P}$ in the medial case.}

\begin{proof}
  First, let's worry about upperbounds and in order to do so   lets consider the behaviour of Algorithm~\ref{alg:buildCore} in the medial case. Actually Algorithm~\ref{alg:buildCore} runs in $F\mathsf{P}$
using an oracle in $\mathsf{NP}$ for checking whether
$\varphi \longrightarrow \varphi''$ holds. Analogously, in the general case, it runs in $F\mathsf{EXP}$
using an oracle in $\mathsf{NP}$. In the agile case, everything becomes polynomial as the size of $\varphi$ and the number of its variables are bounded.
About the lower bounds instead let $\textsc{cff-core}$ be the $\mathsf{DP}$-complete problem~\cite{DBLP:journals/tods/FaginKP05}:
{\em given a pair $(\varphi',\varphi)$ of constants-free formulas such that $\mathit{atm}(\varphi') \subseteq \mathit{atm}(\varphi)$, is \mbox{$\varphi' \simeq \mathit{core}(\varphi)$}?}
%
%
Let $\mathbb{NT}$ be all constant-free formulas in $\WCQ$ closed under $\top$, and $\textsc{nt-core}$ be the analogous of $\textsc{cff-core}$ for $\mathbb{NT}$ formulas.
Indeed, $\textsc{nt-core}$ remains $\mathsf{DP}$-hard. To show this last statement first notice that exploiting proof of Theorem 4.2 in~\cite{DBLP:journals/tods/FaginKP05} we can consider as input for $\textsc{cff-core}$ formulas over a single binary relation, now consider the following mapping $(\varphi',\varphi)\mapsto(\varphi''',\varphi'')$, where $\varphi,\varphi',\varphi''$ and $\varphi'''$ are all constants-free formulas, moreover $\mathit{atm}(\varphi') \subseteq \mathit{atm}(\varphi)$, $\mathit{atm}(\varphi''') \subseteq \mathit{atm}(\varphi'')$, $\varphi$ and $\varphi'$ are over a single binary relation, call it $r$, and both $\varphi'''$ and $ \varphi''$ are $\mathbb{NT}$ formulas. To construct $ \varphi''$ starting from $ \varphi$ (respectively $ \varphi'''$ starting from $ \varphi'$), consider the fresh relation name $s$ and without loss of generality say that $x\not\in\Dom_{\mathit{atm}(\varphi)}$, then we can construct $ \varphi''= x\leftarrow \bigwedge_{\alpha \in \mathit{atm}(\varphi)} \alpha\wedge\bigwedge_{y\in\Dom(\varphi)}s(x,y)\wedge\bigwedge_{z\in\Dom(\varphi)\cup\{x\}}\top(z)$ and $ \varphi'''= x\leftarrow \bigwedge_{\alpha \in \mathit{atm}(\varphi')} \alpha\wedge\bigwedge_{y\in\Dom(\varphi')}s(x,y)\wedge\bigwedge_{z\in\Dom(\varphi')\cup\{x\}}\top(z)$. It is straightforward to notice that $\varphi'$ is the core of $\varphi$ if and only if $\varphi'''$ is the core of $\varphi''$, in fact on one hand say that $\varphi'$ is the core of $\varphi$ and say by contradiction that  $\varphi'''$ is not the core of $\varphi''$, since by construction $\varphi'''\mapsto\varphi''$ the only way possible is that it exists $\tilde{\varphi}\in \mathbb{NT}$ such that $\mathit{atm}(\tilde{\varphi})\subset\mathit{atm}(\varphi''')$ and $\tilde{\varphi}$ is the core of $\varphi'''$ (that coincide by definition with the core of $\varphi''$). Let $\alpha$ be an atom in $\mathit{atm}(\varphi''')\setminus\mathit{atm}(\tilde{\varphi})$, such an atom has to exists in order to have $\mathit{atm}(\tilde{\varphi})\subset\mathit{atm}(\varphi''')$, by construction the only 3 possibility for this atom are:\begin{itemize}
    \item $\alpha=r(z,z')$ for some $z,z'\in\Dom(\varphi''')$.
    \item $\alpha=s(x,z)$ for some $z\in\Dom(\varphi''')$.
    \item $\alpha=\top(z)$ for some $z\in\Dom(\varphi''')$.
\end{itemize} now we will prove that for any of this 3 possibility we will end up with a contradiction. First say that $\alpha$ is of the form $r(z,z')$ for some $z,z'\in\Dom(\varphi''')$ then by definition it exist a homomorphism $h$ from $\mathit{atm}(\varphi''')$ to $\mathit{atm}(\varphi''')$ such that at least 1 of the following two does hold: \begin{itemize}
    \item $h(z)\neq z$.
    \item $h(z')\neq z'$.
\end{itemize}
Say that $h(z)\neq z$, the other implication can be solved by using a similar argument, then we could use the same homomorphism on $ h $ on the variables present in $ \varphi'$ and thus obtain that $ \tilde{\varphi'} $ exists such that $ \mathit{atm}(\tilde{\varphi'})\subset\mathit{atm}(\varphi') $ thus contradicting the assumption that $ \varphi' $ is a core. All other remaining possibilities give rise to the same contradiction following the same construction. On the other hand say that $\varphi'$ is not the core of $\varphi$, then obviously $\varphi'''$ is not the core of $\varphi''$. Now we show that for each $\varphi \in \mathbb{NT}$, it is possible to construct in $F\mathsf{P}$ a pair $(\kb, \unit)$ such that both \mbox{$\varphi$ $\simeq$ $\mathit{can}(\unit,\kb)$} and $\mathit{core}(\varphi)$ $\simeq$ $\mathit{core}(\unit,\kb)$ hold. In order to do so starting from $\varphi \in \mathbb{NT}$ be $k-$ary, without loss of generality consider its free variables to be $x_1,\ldots,x_k$, and consider the following dataset $\db=\{p(cz_1,\ldots,cz_n)~:~p(z_1,\ldots,z_n)\in\mathit{atm}(\varphi)\}\cup\{p(\mathsf{alias},\ldots,\mathsf{alias})~:~p(z_1,\ldots,z_n)\in\mathit{atm}(\varphi)\}$, $\unit=\{\langle cx_1,\ldots,cx_k\rangle,\langle\mathsf{alias},\ldots,\mathsf{alias}\rangle\}$, as a last thing we just have to define our summary selector $\varsigma$ as the following function $\varsigma(\db,\tilde{\tau})$ that return $\db$ whatever the input tuple $\tilde{\tau}$ will be. The really important thing to note in this construction is that the initial database domain is bipartite, in particular there exist no $p(z_1,\ldots,z_n)\in\db$ such that $z_i=\mathsf{alias}$ for some $i\in[n]$ and then there exist $j\in[n]$ such that $z_i\neq z_j$. This important property tells us a priori an important thing about the construction of the canonical explanation of the unit we built, in particular the set $\Genes$ restricted to any connected part of $P$ that contains $d_s$ with $s=cx_i,\mathsf{alias}$ with $i\in[k]$ will always be empty and this tells us we will have no constants at all in the set of atoms of $\mathit{can}(\unit,\kb)$, moreover $\mathit{atm}(\mathit{can}(\unit,\kb))$ will be isomorphic to $\mathit{atm}(\phi)$ by construction, to note this the simplest way is to see that each constant of the direct product that is constructed using the algorithm provided in the main paper will have the form $d_{c,\mathsf{alias}}$ for some constant $c$ present in the union of the  connected parts of the dataset containing the constants $cx_1,\ldots,cx_k$ but by construction this is isomorphic to the starting formula.
We can now reduce $\textsc{nt-core}$ to $\textsc{core}$ in the medial case. From a pair $(\varphi',\varphi)$ of formulas in $\mathbb{NT}$ do: $(i)$ construct in $F\mathsf{P}$ a pair $(\kb,\unit)$ such that $\varphi \simeq \mathit{can}(\unit,\kb)$; $(ii)$ construct $\mathit{core}(\unit,\kb)$; and $(iii)$ check in $\mathsf{NP}$ whether $\varphi' \simeq \mathit{core}(\unit,\kb)$. If $\textsc{core}$ was in $F\mathsf{P}$ in the medial case, then step $(ii)$ would also be in $F\mathsf{P}$ and, hence, $\textsc{nt-core}$ would be in $\mathsf{NP}$, which is impossible unless \mbox{$\mathsf{NP} = \mathsf{coNP}$.}
\end{proof}

From this moment on, the results presented will refer to the reductions present in the main paper, for the convenience of the reader the constructions will first be rewritten briefly and then commented on.
\medskip

\noindent{\bf Theorem 6.} {\em 	\textsc{def} is  $\mathsf{coNEXP}$-complete, $\mathsf{coNP}$-complete and in $\mathsf{P}$ in the general, medial and agile case, respectively. In particular, for each $k>0$, LBs hold already for \textsc{def}$_k$.}
\begin{proof}

As for the upper bound from algorithm~\ref{alg:NotDef} we may notice that $\neg\textsc{def}$ is in $\mathsf{NEXP}$\,/\,$\mathsf{NP}$ in the general\,/\,medial case respectively.
Indeed, line 1 runs in exponential\,/\,polynomial time, lines 2-3 run in
nondeterministic exponential\,/\,polynomial time, and 
lines 4-5 run in exponential\,/\,polynomial time depending on which of the two cases we are considering. 
For the agile case, we first replace  {\bf guess} by  {\bf for} in lines 2-3, and after this transformation the algorithm will work in $\mathsf{P}$.
The following reduction is constructed in such a way as to preserve our assumptions, we will therefore give a single construction that will vary only for the nature of its input, as is natural the general case will be reduced starting from the general case of the starting problem while the medial case from the medial case of the starting problem.
As for the lower bounds instead we will show that for each $k>0$, \mbox{\textsc{ess}$_1$ $\leq$ $\neg$\textsc{def}$_k$} via the mapping $(\kb, \unit, \tau) \mapsto (\kb', \unit^k)$,
where 
$\kb$ $=$ $(\db,\varsigma)$, 
$|\unit| > 1$, 
\mbox{$\kb'$ $=$ $(\db',\varsigma')$,}
$\db'$ $=$ $D$ $\cup$ $\mathit{tw}(\Dom_D,k)$ $\cup$ $\mathit{fc}(\Dom_{\unit\cup \{\tau\}})$,
$\varsigma'$ is s.t. $\varsigma'(D',\tau')$ $=$ $\varsigma(D,\tau_0)$ $\cup$ 
$\mathit{tw}(\Dom_{\tau_0},k)$ $\cup$ $\mathit{fc}(\Dom_{\unit\cup \{\tau\}})$, $\tau' \in \Dom_{D'}^k$ and $\tau_0$ $=$ $\mathit{off}(\tau',\mathsf{dummy})$.
First lets notice that in our context given a characterization for $\unit$ with respect to $\kb'$, call it $\varphi$, then we can easily construct a characterization for $\unit^k$ with respect to $\kb'$ of the form $x_1,\ldots,x_k\leftarrow \bigwedge_{s\in[2..k]}\top(x_s),\mathsf{twin}_s(x_1,x_s) \wedge\bigwedge_{\alpha \in \mathit{atm}(\varphi)}\alpha$, vice versa say that you have  a characterization for $\unit^k$, then by renaming every occurrence of the free variables other then the first one we will end up with a characterization for $\unit$, having asserted this we can equivalently solve the problem by concentrating on the unary case. Let's start by assuming that $\tau\not\in\mathit{ess}(\unit,\kb) $ then adding the atom $\mathsf{focus}(x_1)$ to the characterization we had for $\unit$ with respect to $\kb$ we will end up with a formula that defines $\unit$ with respect to $\kb'$, this is because the atom we just added will keep away every element of $\mathit{ess}(\unit,\kb)$ other than the element of $\unit$ itself. On the other hand if $\tau\in\mathit{ess}(\unit,\kb) $ then it will also hold that $\tau\in\mathit{ess}(\unit,\kb') $, this can be observed directly by constructing a characterization for the unit which will be nothing more than the previous characterization with the addition of some atoms of the form $\mathsf{focus}(y)$ together with $\mathsf{focus}(x_1)$, and again this characterization will also interpret $\tau$. What remains to be shown is that the new summary selector can be built without making complexity leaps and that its execution time turns out to be polynomial with respect to the input received, but this is evident as we only have to do a linear scan of the input that allows us to reconstruct the initial dataset and build the new tuple on which to then relaunch the old summary selector, this second part will by hypothesis run in polynomial time, and finally we will add the other elements always doing a linear scan of the input.
\end{proof}

The result of the following theorem is extremely useful in the course of the other proofs.

\medskip

\noindent{\bf Theorem 7.} {\em 	\textsc{ess} is $\mathsf{NEXP}$-complete, $\mathsf{NP}$-complete and in $\mathsf{P}$ in the general, medial and agile case, respectively. 
In particular, LBs  hold for \textsc{ess}$_k$ with $k>0$ even if $|\unit|>1$.}
\begin{proof}
   For what concernes the upper bounds one has just to notice that we can use 
a simple variant of Algorithm~\ref{alg:NotDef} created by removing line $2$.
%
For what concerns the lower bounds instead we will give two distinct constructions, one for the general case and one for the medial one. Let \textsc{php-sb} be the $\mathsf{NEXP}$-complete problem: {\em given a sequence \mbox{${\bf s}$ $=$ $I_1,..,I_{m+1}$} of instances over the binary relation $\mathsf{br}$ with \mbox{$\Dom_{I_i} \cap \Dom_{I_{j\neq i}}\!=\! \emptyset$,} is there any homomorphism from \mbox{$I_1\! \otimes\!...\!\otimes\!I_{m}$} to $I_{m+1}$?} 
Let \textsc{{\footnotesize3}-col} be the $\mathsf{NP}$-complete problem: {\em given a graph \mbox{$G$ $=$ $(V,E)$,} is there a map \mbox{$\lambda: V \rightarrow \{\mathtt{r},\mathtt{g},\mathtt{b}\}$} such that \mbox{$\{u,v\}\!\in\!V$} implies \mbox{$\lambda(u)\!\neq\!\lambda(v)$?}}
Please note that \textsc{php-sb} is $\mathsf{NEXP}$-complete when $m$ is unbounded in the following we will consider $m>1$.
In the general case, for each \mbox{$k\!>\!0$,}
by adapting a technique of \citeauthor{DBLP:conf/icdt/CateD15}~(\citeyear{DBLP:conf/icdt/CateD15}), we show that \mbox{\textsc{php-sb} $\leq$ \textsc{ess}$_k$} via the map
{\bf s} $\mapsto$ $(\kb, \unit^k, \langle \mathsf{a}_{m+1}\rangle^k)$, where
\mbox{$\kb$ $=$ $(\db,\varsigma)$,} 
\mbox{$\db$ $=$ $\db_1 \cup...\cup \db_4$,}
$\db_1$ $=$ $I_1 \cup ... \cup I_{m+1}$,
$\db_2$ $=$ $\{\top(c):c \in \Dom_{\db_1}\}$,
$\db_3$ $=$ $\{\top(\mathsf{a}_i),\mathsf{br}(\mathsf{a}_{i},c):c \in \Dom(I_i)\}_{i \in[m+1]}$,
$\db_4$ $=$ $\mathit{tw}(\{\mathsf{a}_1,...,\mathsf{a}_{m+1}\},k)$,
$\varsigma$ always selects $D$,
$\unit$ $=$ $\{\langle \mathsf{a}_1 \rangle,..,\langle \mathsf{a}_m \rangle\}$. 
First lets notice that in our context given a characterization for $\unit$, call it $\varphi$, then we can easily construct a characterization for $\unit^k$  of the form $x_1,\ldots,x_k\leftarrow \bigwedge_{s\in[2..k]}\top(x_s),\mathsf{twin}_s(x_1,x_s) \wedge\bigwedge_{\alpha \in \mathit{atm}(\varphi)}\alpha$, vice versa say that you have  a characterization for $\unit^k$, then by renaming every occurrence of the free variables other then the first one we will end up with a characterization for $\unit$, according to this we can simply focus on the unary case and then add this detail later. %
Let $S$ denote $I_1 
		\otimes\ldots\otimes
		I_m$.
  In
		one direction, if $S \mapsto I_{m+1}$ then, by homomorphism preservation, any conjunctive formula that interprets the elements of $\unit$ will also interpret $\langle a_{m+1}\rangle$.
		On the other hand, if $ S \not\mapsto I_{m+1}$, then we can construct a formula $\varphi\in\WCQ$ that interprets $\unit$ but not $\langle a_{m+1}\rangle$.
  To this aim first consider $\varphi_S$ to be the characteristic formula of $S$.
  Then take $\varphi$ as $x\leftarrow \varphi_S\bigwedge_{y\in\Dom_{\varphi_S}}\mathsf{br}(x,y)$.
 By construction, $\varphi$ interprets all the elements of $\unit$ but does not interpret $\langle a_{m+1} \rangle$.
 In order to see this, on the one hand consider, for each $i\in[m]$,  $h_i$ as the projection onto the $i$-th component. Obviously each of these is a homomorphism from $ I_1 
		\otimes\ldots\otimes
		I_m \mapsto I_i$ and in particular this shows that $\varphi_S$ is true when evaluated on every $ I_i$ with $i\in[m]$, implying $\langle a_{i}\rangle\in\mathit{inst}(\varphi,\kb)\forall i \in [m]$. On the other hand, suppose by contraddiction that there exist a homomorphism $h$ $\mathit{atm}(\varphi)$ to $\varsigma(a_{m+1})$ such that $\langle h(x) \rangle=\langle a_{m+1} \rangle $, then the same homomorphism $h$ limited to variables other than $x$ would show the existence of a homomorphism $h'$ from $ I_1 
		\otimes\ldots\otimes
		I_m$ to $I_{m+1}$ and this is absurd.
What remains to be shown is that the new summary selector can be built without making complexity leaps and that its execution time turns out to be polynomial with respect to the input received, but this is straightforward since the summary selector has always to return the entire dataset in input.

For each \mbox{$k\!>\!0$,} in the medial case, \mbox{\textsc{{\footnotesize3}-col} $\leq$ \textsc{ess}$_k$} via the map
$G$ $\mapsto$ 
$(\kb, \unit^k, \langle \mathsf{b}_1\rangle^k)$, where  \mbox{$\kb$ $=$ $(\db,\varsigma)$,} 
$\db$ $=$ $\db_1$ $\cup...\cup$ $\db_5$,
$\db_1$ $=$ $\{\mathsf{arc}(u_i,v_i),$ $\mathsf{arc}(v_i,u_i),$ $:\{u,v\}\in E \wedge i\in[2]\}$,
$\db_2$ $=$ $\{\top(c):c \in \Dom_{\db_1}\}$,
$\db_3$ $=$ $\{\top(\mathsf{b}_i),\mathsf{arc}(\mathsf{b}_i,\mathsf{b}_z):i,z\in[4],i\neq z\}$,
$\db_4$ $=$\! $\{\top(\mathsf{a}_i),\mathsf{arc}(\mathsf{a}_i,v_i),\mathsf{arc}(v_i,\mathsf{a}_i)\!:\!v\in\!V \wedge i\!\in\![2]\}$, $\db_5$ $=$ $\mathit{tw}(\Dom_{\db_3}\cup\Dom_{\unit},k)$, 
$\varsigma$ always selects $D$ and
$\unit = \{\langle \mathsf{a}_1 \rangle,\langle \mathsf{a}_2 \rangle\}$.	    
	Again lets notice that in our context given a characterization for $\unit$, call it $\varphi$, then we can easily construct a characterization for $\unit^k$  of the form $x_1,\ldots,x_k\leftarrow \bigwedge_{s\in[2..k]}\top(x_s),\mathsf{twin}_s(x_1,x_s) \wedge\bigwedge_{\alpha \in \mathit{atm}(\varphi)}\alpha$, vice versa say that you have  a characterization for $\unit^k$, then by renaming every occurrence of the free variables other then the first one we will end up with a characterization for $\unit$, having asserted this we can equivalently solve the problem by concentrating on the unary case.
  Let $K_4=(V_4,E_4)$ denote the complete graph of order four.
Note that $\db_3$ translates $K_4$ in a dataset whereas $\db_1\cup\db_4$ translates two disjoint objects isomorphic to $G^a$, call them $G^{a_1}$ and $G^{a_2}$, in a dataset.
Let $\tilde{a}\in\{a_1,a_2\}$.
On the one hand say that $G$ is 3-colourable, then, from Lemma~\ref{lem:color}, $G^{\tilde{a}}$ is 4-colourable, so it exists a homomorphism $h$ from $G^{\tilde{a}}$ to $K_4$, this homomorphism will map $\tilde{a}$ to a certain element of $V_4$ and, since element from $V_4$ are equal up to isomorphism, w.l.o.g. $h(\tilde{a})=\mathsf{b}_1$. Let now $\varphi$ be a characterization for $\mathit{ess}(\unit,\kb)$, this means that it exists a homomorphism $h_1: \mathit{atm}(\varphi)\mapsto dom(\kb)$ such that $h_1(x)=\mathsf{\tilde{a}}$ but then by homomorphism composition $h_2: \mathit{atm}(\varphi)\mapsto dom(\kb)$ defined as $h\circ h_1$ is another homomorphism such that $h_2(x)=\mathsf{b}_1$  and so  $\langle \mathsf{b}_1\rangle\in \mathit{ess}(\unit,\kb)$.
On the other hand now say that $\langle \mathsf{b}_1\rangle\in \mathit{ess}(\unit,\kb)$, this means that $\mathit{can}(\mathit{ess}(\unit,\kb))$ interprets $V_4$, in particular this gives raise to a homomorphism from $G^{\tilde{a}}$ to $K_4$, but then, from Lemma~\ref{lem:color}, $G$ is 3-colourable.

As before what remains to be shown is that the new summary selector can be built without making complexity leaps and that its execution time turns out to be polynomial with respect to the input received, but again this is straightforward because also in this case the summary selector has always to return the entire dataset in input.
\end{proof}
The following result among the various ones is the one that foresees one of the most linear constructions, we will do nothing but build an almost indistinguishable copy of an object, the first tuple of our input unit to be precise.

\medskip

\noindent{\bf Theorem 8.} {\em 	\textsc{sim} is $\mathsf{NEXP}$-complete, $\mathsf{NP}$-complete and in $\mathsf{P}$ in the general, medial and agile case, respectively. In particular, for each $k>0$, LBs hold already for \textsc{sim}$_k$.}
\begin{proof}
For what concerns the upper bounds it's easy to notice that an input for ${\textsc{sim}}$ is $\mathsf{accepted}$ if, and only if, it is $\mathsf{accepted}$ both for \textsc{gad{\scriptsize1}} and  \textsc{gad{\scriptsize2}}.
The following reduction is constructed in such a way as to preserve our assumptions, we will therefore give a single construction that will vary only for the nature of its input, as is natural the general case will be reduced starting from the general case of the starting problem while the medial case from the medial case of the starting problem. Now lets face the lower bounds for this problem. It holds that \textsc{ess}$_1$ $\leq$ \textsc{sim}$_k$ via the map $(\kb, \unit, \tau) \mapsto (\kb', \unit^k,\tau^k,\langle \mathsf{alias}\rangle^k)$, where $\kb$ $=$ $(D,\varsigma)$,
$\unit$ $=$ $\{\tau_1,..,\tau_m\}$,
$m\!>\!1$, 
$\tau_1$ is of the form $\langle c \rangle$,
$\kb'$ $=$ $(D',\varsigma')$,
\mbox{$D'$ $=$ $\bar{D}$ $\cup$ $\mathit{tw}(\Dom_{\bar{D}},k)$,}
\mbox{$\bar{D}$ $=$ $\mathit{double}(D,c)$,}
selector $\varsigma'$ is such that $\varsigma'(D',\tau')$ $=$ $\varsigma(D,\tau_0)$ $\cup$
$\mathit{tw}(\Dom_{\{\tau_0\}},k)$, $\tau' \in \Dom_{D'}^k$ and $\tau_0$ $=$ $\mathit{off}(\tau',c)$. 
First lets notice that in our context given a characterization for $\unit$, call it $\varphi$, then we can easily construct a characterization for $\unit^k$  of the form $x_1,\ldots,x_k\leftarrow \bigwedge_{s\in[2..k]}\top(x_s),\mathsf{twin}_s(x_1,x_s) \wedge\bigwedge_{\alpha \in \mathit{atm}(\varphi)}\alpha$, vice versa say that you have  a characterization for $\unit^k$, then by renaming every occurrence of the free variables other then the first one we will end up with a characterization for $\unit$, having asserted this we can equivalently solve the problem by concentrating on the unary case. The first thing the reader has to notice is that $\langle\mathsf{alias}\rangle\in\mathit{ess}(\unit,\kb')$ by construction, in fact say that $h$ is a $\C$-homomorphism from $\mathit{atm}(\mathit{can}(\mathit{ess}(\unit,\kb)))$ to $\varsigma'(\db',\langle c \rangle)$ such that $ x_1  \mapsto c$, please note that this $\C$-homomorphism must exist by hypotheses, then we can construct another $\C$-homomorphism from $\mathit{atm}(\mathit{can}(\mathit{ess}(\unit,\kb)))$ to $\varsigma'(\db',\langle \mathsf{alias} \rangle)$, call it $\tilde{h}$, such that $ x_1  \mapsto \mathsf{alias}$ as follows $\tilde{h}(y)= h(y)$ if $y\neq x_1$ and $\tilde{h}(x_1)= \mathsf{alias}$. That been said the problem collapse to wheter or not $\tau\in\mathit{ess(\unit,\kb)}$ and if that is the case then we will end up with a yes answer since the same $\C-$homomorphism that was a witness for $\tau$ to be in $\mathit{ess(\unit,\kb)}$ is still a witness for $\tau$ to be in $\mathit{ess(\unit,\kb')}$ and viceversa having $\C-$homomorphism that is a witness for $\tau$ to be in $\mathit{ess(\unit,\kb')}$ leads us to a witness for $\tau$ to be in $\mathit{ess(\unit,\kb)}$.
 What remains to be shown is that the new summary selector can be built without making complexity leaps and that its execution time turns out to be polynomial with respect to the input received, but this is evident as we only have to do a linear scan of the input that allows us to reconstruct the initial dataset and build the new tuple on which to then relaunch the old summary selector, this second part will by hypothesis run in polynomial time, and finally we will add the other elements always doing a linear scan of the input.
\end{proof}

Unlike the previous theorem, here we will have to act differently, the idea behind it is to provide the whole dataset, except for some specially chosen elements, access to a new predicate that is therefore able to separate the family units that can be defined in two parts, those that can be characterized with formulas that have access to such information and those that instead have characterizations without such information.
\medskip

\noindent{\bf Theorem 9.} {\em 	\textsc{inc} is $\mathsf{coNEXP}$-complete, $\mathsf{coNP}$-complete and in $\mathsf{P}$ in the general, medial and agile case, respectively. In particular, for each $k>0$, LBs hold already for \textsc{inc}$_k$.}
\begin{proof}
  For what concerns the upper bounds it's easy to notice that an input for ${\textsc{inc}}$ is  {\em $\mathsf{accepted}$} if, and only if, it is  {\em $\mathsf{accepted}$} both for $\neg$\textsc{gad{\scriptsize1}} and  $\neg$\textsc{gad{\scriptsize2}}.
The following reduction is constructed in such a way as to preserve our assumptions, we will therefore give a single construction that will vary only for the nature of its input, as is natural the general case will be reduced starting from the general case of the starting problem while the medial case from the medial case of the starting problem.
For what concerns the lower bounds it holds that \textsc{ess}$_1$ $\leq$ $\neg$\textsc{inc}$_k$ via the following map: $(\kb, \unit, \tau) \mapsto (\kb', \unit^k,\tau^k,\langle \mathsf{alias}\rangle^k)$, where \mbox{$\kb$ $=$ $(D,\varsigma)$,}
\mbox{$\unit$ $=$ $\{\tau_1,..,\tau_m\}$,}
$m\!>\!1$,
$\tau_1$ is of the form $\langle c \rangle$,
\mbox{$\kb'$ $=$ $(\db',\varsigma')$,} 
\mbox{$\db'$ $=$ $D$ $\cup$
$\mathit{tw}(\Dom_{\bar{D}},k)$ $\cup$ $\mathit{fc}(\Dom_\db)$,}
$\bar{D} = \mathit{double}(D,c)$,
$\varsigma'$ is such that $\varsigma'(D',\tau')$ $=$ $\varsigma(D,\tau_0)$ 
$\cup$ $\mathit{tw}(\Dom_{\{\tau_0\}},c)$
$\cup$
$\mathit{fc}(\Dom_\db)$, $\tau' \in \Dom_{D'}^k$ and $\tau_0$ $=$ $\mathit{off}(\tau',c)$. First lets notice that in our context given a characterization for $\unit$, call it $\varphi$, then we can easily construct a characterization for $\unit^k$  of the form $x_1,\ldots,x_k\leftarrow \bigwedge_{s\in[2..k]}\top(x_s),\mathsf{twin}_s(x_1,x_s) \wedge\bigwedge_{\alpha \in \mathit{atm}(\varphi)}\alpha$, vice versa say that you have  a characterization for $\unit^k$, then by renaming every occurrence of the free variables other then the first one we will end up with a characterization for $\unit$, having asserted this we can equivalently solve the problem by concentrating on the unary case.
By costruction we know that $\langle \mathsf{alias} \rangle$ for sure is not in $\mathit{ess}(\unit,\kb')$ and this is because all elements in unit "have access" to $\mathsf{focus}$, this let us know that a characterization for $\unit$ will use that predicate and so it will not interpret $\langle\mathsf{alias}\rangle$. That been said the problem collapse to wheter or not $\tau\in\mathit{ess(\unit,\kb)}$ or not and if that is the case then we will end up with a no answer since the same $\C-$homomorphism that was a witness for $\tau$ to be in $\mathit{ess(\unit,\kb)}$ can be enlarged to a witness for $\tau$ to be in $\mathit{ess(\unit,\kb')}$ and viceversa having $\C-$homomorphism that is a witness for $\tau$ to be in $\mathit{ess(\unit,\kb')}$ leads us to a witness for $\tau$ to be in $\mathit{ess(\unit,\kb)}$.
 What remains to be shown is that the new summary selector can be built without making complexity leaps and that its execution time turns out to be polynomial with respect to the input received, but this is evident as we only have to do a linear scan of the input that allows us to reconstruct the initial dataset and build the new tuple on which to then relaunch the old summary selector, this second part will by hypothesis run in polynomial time, and finally we will add the other elements always doing a linear scan of the input.
\end{proof}
The following result is interesting because it provides a non-classical example of a complete problem for the two classes under consideration, usually in fact such problems appear as made up of pairs of strings that can be interpreted separately.

\medskip

\noindent{\bf Theorem 10.} {\em \textsc{prec} is $\mathsf{DEXP}$-complete, $\mathsf{DP}$-complete and in $\mathsf{P}$ in the general, medial and agile case, respectively. In particular, for each $k>1$, LBs hold already for \textsc{prec}$_k$.}
\begin{proof}
 For what concerns the upper bounds it's easy to notice that an  input for ${\textsc{prec}}$ is  {\em $\mathsf{accepted}$} if, and only if, it is  {\em $\mathsf{accepted}$}  both for \textsc{gad{\scriptsize1}} and  $\neg$\textsc{gad{\scriptsize2}}.
The following reduction is constructed in such a way as to preserve our assumptions, we will therefore give a single construction that will vary only for the nature of its input, as is natural the general case will be reduced starting from the general case of the starting problem while the medial case from the medial case of the starting problem.
Lets now talk about lower bounds.  For each $L\in \mathsf{DEXP}$ (resp., $\mathsf{DP}$), \mbox{$L$ $\leq$ \textsc{prec}$_k$} in the general (resp., medial) case, via the mapping \mbox{$w \mapsto (\kb, \unit, \tau', \tau'')$,} 
where
$w$ is any string over the alphabet of $L$,
$L$ $=$ $L_1$ $\cap$ $L_2$,
$L_1$ $\in$ $\mathsf{NEXP}$ (resp., $\mathsf{NP}$),
$L_2$ $\in$ $\mathsf{coNEXP}$ (resp., $\mathsf{coNP}$),
$\rho_1$ is a reduction from $L_1$ to ${\textsc{ess}}_{1}$ in the general (resp., medial) case,
$\rho_2$ is a reduction from $L_2$ to $\neg{\textsc{ess}}_{k-1}$ in the general  (resp., medial) case,
each \mbox{$\rho_i(w)$ $=$ $(\kb_i, \unit_i, \tau_i)$,}
each \mbox{$\kb_i$ $=$ $(\db_i,\varsigma_i)$,}
w.l.o.g. each constant and predicate different from $\top$ in $\db_1$ (resp., $\db_2$) has $a$- (resp., $b$-) as a prefix;
each \mbox{$\unit_i$ $=$ $\{\tau_{i,1},..,\tau_{i,m_i}\}$,}
\mbox{$\kb$ $=$ $(\db,\varsigma)$,}
\mbox{$\db$ $=$ $\db_1\cup\db_2$,}
$\varsigma$ is such that
$\varsigma(\db,\tau)=\varsigma_1(\db_1,\tau|_{\Dom(\db_1)}) \cup \varsigma_2(\db_2,\tau|_{\Dom(\db_2)})$,
$\unit$ $=$ $\unit_1\times\unit_2$,
$\tau'$ $=$ $\tau_1\circ\tau_{2,1}$,
$\tau''$ $=$ $\tau_{1,1}\circ\tau_2$,
each $\tau|_{\Dom(\db_i)}$ is obtained from $\tau$ by removing the terms not in $\Dom(\db_i)$. The main idea of the construction is that the final knowledgebase that we construct is actually bipartite in two distinct sense, it is not in fact bipartite just from the domain point of view but in particular the predicate $\top$ is the only one shared by both type of constants, apart from that we may say that constants coming from the first reduction are not just totally different from constant occurring in the second reduction but also they appears in atoms with totally different predicate names. 
Noting this fact gives us the possibility to say that the explanation we will construct will be necessary separable in the sense of Definition~\ref{def:sep}.
From this everything will follow by Corollary \ref{cor:bipQuery}. In fact we do know by construction that both $\tau_{1,1}\in\mathit{ess}(\unit_1,\kb_1)$ and $\tau_{2,1}\in\mathit{ess}(\unit_2,\kb_2)$ hold and so the input for the new problem we have constructed will return a positive answer if and only if both $\tau_1\in\mathit{ess}(\unit_1,\kb_1)$ and $\tau_2\not\in\mathit{ess}(\unit_2,\kb_2)$ hold.
What remains to be shown is that the new summary selector can be built without making complexity leaps and that its execution time turns out to be polynomial with respect to the input received, but this is just a combination of the running of two distinct procedures that we do know by hypotheses run in polynomial time, moreover we are able to easily reconstruct both $\db_1$ and $\db_2$ via a linear scan of the input.
\end{proof}

\end{document}